\documentclass[noinfoline]{imsart}

\usepackage[utf8]{inputenc}
\usepackage[T1]{fontenc}
\usepackage[english]{babel}
 
\usepackage{cite}			%meilleures citations

\usepackage{dsfont}			% double stroke fonts
\usepackage{stmaryrd}
\usepackage{bbm}
\usepackage{mathrsfs}
\usepackage{lmodern}

\usepackage{relsize}

\usepackage{fancyhdr}
\usepackage{titling}

\usepackage{pstricks}
\usepackage{pifont}
\usepackage{enumerate}
\usepackage{graphicx}
\usepackage[all]{xy}
\usepackage{multicol}
\usepackage{longtable}
\usepackage{amsmath,amsfonts,amssymb,amstext}
\usepackage[amsmath,amsthm,thmmarks,hyperref]{ntheorem}
\usepackage{nicefrac}
\usepackage{fancyvrb}

\usepackage{varioref}
\usepackage[unicode,bookmarks,colorlinks=true,citecolor=red,linkcolor=black]{hyperref}
\usepackage[all]{hypcap}
\usepackage{cleveref}

\addto\extrasenglish{

}

\newcommand{\thmref}[1]{\ref{#1} (page \pageref{#1})}

\DeclareMathOperator*{\bigboxplus}{\;\text{\huge$\boxplus$}\;}

\DeclareMathOperator{\supp}{supp}

\newcommand{\ds}{\displaystyle}

\renewcommand{\C}[1]{\mathscr{#1}}
\newcommand{\F}[1]{\mathfrak{#1}}
\newcommand{\B}[1]{\mathds{#1}}
\newcommand{\ov}[1]{\overline{#1}}
\newcommand{\und}[1]{\underline{#1}}
\newcommand{\wt}[1]{\widetilde{#1}}
\newcommand{\Ct}[1]{\wt{\C{#1}}}
\newcommand{\wh}[1]{\widehat{#1}}

\renewcommand{\phi}{\varphi}

\renewcommand{\leq}{\leqslant}
\renewcommand{\geq}{\geqslant}

\newcommand{\indi}{\mathbbm{1}}

\newcommand{\ud}{\mathrm{d}}

\newcommand{\quotient}[2]{\nicefrac{\mathlarger{#1}}{\mathlarger{#2}}}

\numberwithin{equation}{section}

\theoremstyle{break}
\newtheorem{thm}{Theorem}[section]
\newtheorem{lemma}[thm]{Lemma}
\newtheorem{prop}[thm]{Proposition}

\newtheorem{dfn}{Definition}[section]

\theoremheaderfont{\itshape}

\crefname{dfn}{definition}{definition}
\crefname{dfn}{Definition}{Definition}
\crefname{prop}{proposition}{proposition}
\crefname{prop}{Proposition}{Proposition}

\allowdisplaybreaks

\begin{document}

\begin{frontmatter}
\title{Toric grammars : a new statistical 
approach to natural language modeling}
\runtitle{Toric grammars}
\begin{aug}
\author{Olivier Catoni}
\author{Thomas Mainguy}
\runauthor{O. Catoni and T. Mainguy ~/~ \today~ }
\end{aug}

\begin{abstract}
We propose a new statistical model for 
computational linguistics. Rather than trying to estimate 
directly the probability distribution of a random 
sentence of the language, we define a Markov chain on 
finite sets of sentences with many finite recurrent communicating 
classes and define our language model as the invariant 
probability measures of the chain on each recurrent 
communicating class. This Markov chain, that we call a communication 
model, recombines at each step randomly 
the set of sentences forming its current state, using some 
grammar rules. When the grammar rules are fixed and known in advance 
instead of being estimated on the fly, we can prove supplementary
mathematical properties. In particular, we can prove 
in this case that all states are recurrent states, so that 
the chain defines a partition of its state space into 
finite recurrent communicating classes. We show that our approach
is a decisive departure from Markov models at the sentence level 
and discuss 
its relationships with Context Free Grammars. Although 
the toric grammars we use are closely related to Context 
Free Grammars, the way we generate the language
from the grammar is qualitatively different. 
Our communication model has two purposes. On the one hand, it is 
used to define indirectly the probability distribution of a 
random sentence of the language. On the other hand it can 
serve as a (crude) model of language transmission from 
one speaker to another speaker through the communication 
of a (large) set of sentences.
\end{abstract}


\end{frontmatter}

\section{Introduction to a new communication model}

In\footnote[0]{\today} the well known kernel approach to density estimation on a measurable 
space $\C{X}$, the probability  
distribution $\B{P}$ 
 of a random variable $X \in \C{X}$ 
is estimated from a sample $(X_1, \dots, X_n)$ of $n$ independent copies 
of $X$ as $\frac{1}{n} \sum_{i=1}^n 
k(X_i, \ud x)$, where $k$ is a suitable Markov kernel. This kernel 
estimate can be seen as a modification of the empirical measure 
$\ov{\B{P}} = \frac{1}{n} \sum_{i=1}^n \delta_{X_i}$. 

In the context of natural language modeling at the sentence level, 
$\C{X}$ is the set of sentences, that is the set of sequences of 
words of finite length. 

Finding sensible kernel estimates or sensible parametric models 
in this context is a challenge. Therefore, we propose here another 
route, that we will describe as an alternative way of producing 
a modification of the empirical measure. The idea is to recombine
repeatedly a set of sentences. Let us describe for this a general 
framework, concerned with an arbitrary countable state space $\C{X}$. 

Let $\ds \ov{\C{P}}_n = \biggl\{ \frac{1}{n} \sum_{i=1}^n \delta_{x_i}, 
\, x_i \in \C{X} \biggr\}$ be the set of empirical measures of all possible 
samples of size $n$. Let us consider a parametric family 
$\{q_{\theta}, \theta \in \Theta \}$ of Markov kernels on $\ov{\C{P}}_n$.
Let us assume for simplicity that for any $P \in \ov{\C{P}}_n$, 
the reachable set $\bigl\{ Q \in \ov{\C{P}}_n, \; \sum_{t \in \B{N}} 
q_{\theta}^t \bigl(P,Q)  > 0\bigr\}$ is finite, where $q_{\theta}^t$ is $q_{\theta}$ composed $t$ times with itself, so that for instance 
$q_{\theta}^2(P,Q) = \sum_{P' \in \ov{\C{P}}_n} 
q_{\theta}(P, P') q_{\theta}(P', Q)$. In this case we can define the 
Markov kernel 
\[
\wh{q}_{\theta}(P,Q) = \lim_{k\rightarrow \infty} \frac{1}{k} \sum_{t=1}^k 
q_{\theta}^t(P,Q).
\]
It is such that for any $P \in \ov{\C{P}}_n$, $\wh{q}_{\theta}(P, \cdot)$ is 
an invariant measure of $q_{\theta}$. More generally 
$q_{\theta} \wh{q}_{\theta} = \wh{q}_{\theta} q_{\theta} = \wh{q}_{\theta}$. 
The distribution 
$\wh{q}_{\theta}(P, \cdot) \in \C{M}_+^1 \bigl( \ov{\C{P}}_n \bigr)$ 
induces a marginal distribution $\wh{Q}_{\theta, P}$ on $\C{X}$ through 
the formula
\begin{equation}
\label{eq1.1}
\wh{Q}_{\theta, P} = \sum_{Q \in \ov{\C{P}}_n} \wh{q}_{\theta}(P,Q) 
Q. 
\end{equation}

In this paper, we will be concerned with estimators of the form $\wh{P} = 
\wh{Q}_{\theta, \ov{\B{P}}}$, if $\theta$ is fixed in advance, or of the form 
$\wh{Q}_{\wh{\theta}, \ov{\B{P}}}$, if $\wh{\theta}$ is an estimator 
of the parameter $\theta$ depending also on $\ov{\B{P}}$. 

Another interpretation of our framework is to consider $q_{\theta}$ 
as a communication model. One speaker hears a set of sentences described 
by its empirical distribution $P \in \ov{\C{P}}_n$ (which means 
that he will not make use of the special order in which he has heard 
them). He uses those sentences to learn the corresponding 
language. Then he teaches another speaker what he has learnt 
by outputting another random set of sentences, distributed according 
to $q_{\theta}(P, \cdot)$. The language model (as opposed to the 
communication model $q_{\theta}$), is $\wh{Q}_{\theta, P}$, the 
average sentence distribution along an infinite chain of 
communicating speakers. If we start from a recurrent state $P$, 
and we assume that $\theta$ is known, we obtain a communication 
model where the target sentence distribution $\wh{Q}_{\theta, P}$ 
can be learnt without error from the set of sentences output 
by any involved speaker. Indeed $\wh{Q}_{\theta, P} = 
\wh{Q}_{\theta, Q}$ for any $Q$ in the communicating class of $P$, 
which in this situation is also the reachable set from $P$.

This error free estimation behaviour is desirable for a communication 
model. It tells us that the language can be transmitted from speaker 
to speaker without distortion, a desirable feature in the case of 
a large number of speakers. The model may also account for weak 
stimulus learning, the fact that human beings learn language 
through a limited number of examples compared with the variety 
of new sentences they are able to formulate. Indeed, whereas
the size of the support of $P \in \ov{\C{P}}_n$, the number of sentences 
heard by one speaker, is constant and equal to $n$, the support of the 
language model $\wh{Q}_{\theta, P}$ may be much larger.  
We will actually give a toy example where the number of sentences 
in the language is exponential with $n$.

In the language transmission interpretation, we may evaluate 
the interest of the model by studying whether it can model
a large family of sentence distributions. This richness will 
depend on the number of recurrent communicating classes of 
the communication Markov model $q_{\theta}$, since any 
invariant distribution $\wh{q}_{\theta}(P, \cdot)$ is a convex 
combination of the unique invariant measures supported by 
each recurrent communicating class. The situation is even 
simpler in the case when all $P \in \ov{\C{P}}_n$ are recurrent
states (a fact we will be able to prove in 
our particular model). In this case
$\wh{q}_{\theta}(P, \cdot)$ is the unique invariant measure 
supported by the recurrent communicating class to which $P$ belongs.

In this paper we will focus on the construction and mathematical 
properties of the communication model $q_{\theta}$. We will 
also touch on the estimation problem stated in the opening of this 
introduction by providing some estimator $\wh{Q}_{\wh{\theta}, \ov{\B{P}}}$ 
where $\wh{\theta}(\ov{\B{P}})$ is an estimator of the parameter
computed on the observed sample. However we will leave the mathematical 
properties of this estimator for further studies. We will be content 
with providing some promising preliminary experiments 
computed on a small sample and will share with the reader 
some qualitative explanations of its behaviour. 

The parameter $\theta$ of our model will be a new kind of 
grammar, closely related to Context Free Grammars, but 
used to generate sentences in a different way.

\section{Toric grammars}

Now that we have explained our general 
framework based on a communication Markov kernel $q_{\theta}$ defined 
on empirical distributions, let us come to natural language modeling 
more specifically, and describe a dedicated family of kernels.

Natural language processing in linguistics 
has been using more and more elaborate mathematical tools (a brief presentation of some of them is given by E. Stabler in \cite{Stabler09}). 
The $n$-gram models are widely used, although they fail to grasp the recursive nature of natural languages, and do not use the syntactic properties of sentences. Efforts have been made to improve the performance of these models, 
by introducing syntax, (Della Pietra et al. 1994 \cite{Pietra94}; Roark 2001 \cite{Roark01}; Tan et al. 2012 \cite{Tan12}). One way to do this is to use Context Free Grammars, also named phrase structure grammars, introduced by N. Chomsky as possible models for the logical structure of natural languages (see for example \cite{Chomsky56, Chomsky57, Chomsky65}), and their probabilistic variants (Chi 1999 \cite{Chi99}). Our proposal follows this trend, but with the goal to separate ourselves from classic $n$-grams, 
seeing syntax as equivalence classes between constituents, which we try to discover.

We consider some dictionary of words $D$. Each statistical sample, as explained 
in the introduction, is made of a set of sentences. Each sentence 
is a sequence of words of $D$.
The sentences may be of variable length. 

To simplify notations we will use non normalized empirical measures. 
Thus, the state space of the communication Markov kernel $q_{\theta}$ 
will be
$$
\ov{\C{P}}_n = \biggl\{ \sum_{i=1}^n \delta_{s_i}, s_i \in D^+ \biggr\},
$$
where we introduce the notation
\[
D^+ = \bigcup_{j=1}^{\infty} D^j.
\]

We will call $\ov{\C{P}}_n$ the set of texts of length $n$. 
Let us notice that for us, texts are unordered sets of 
sentences. The question of generating meaningful ordered 
sequences of sentences is also of interest, but will not 
be addressed in this study.

In order to define the communication kernel,
we will describe random transformations on texts,
related to the notion of Context Free Grammars.
Let us start with an informal presentation. 
The communication kernel will perform 
random recombinations of sentences. 

Our point of view is to see a Context Free Grammar as the result of some fragmentation process applied to a set of sentences. Let us explain this on a simple example. Consider the sentence
\[
  \text{\it This is my friend Peter.}
\]
Imagine we would like to represent this sentence as the result of pasting the expression {\it my friend} in its context, because we think language is built by cutting and pasting expressions drawn from some large set of memorized sentences. We can do this by introducing the simple Context Free Grammar
\begin{align*}
  \framebox{$0$} & \rightarrow \text{\it This is } \framebox{$1$} \text{\it ~Peter .}\\
  \framebox{$1$} & \rightarrow \text{\it my friend}
\end{align*}
where we have used numbered framed boxes for non terminal symbols, the start symbol being~\framebox{$0$}. The two rules mean that we can rewrite the start symbol~\framebox{$0$} to obtain the right-hand side of the first rule, and that we can then rewrite the non terminal symbol~\framebox{$1$} as the right-hand side of the second rule. 

Since we want to see the rules of the grammar as the result of some splitting operation, we are going to use more symmetric notations. Instead of considering that we have described our original sentence with the help of two rules and two non terminal symbols~\framebox{$0$} and~\framebox{$1$}, we may as well consider that we have split our original sentence into two new sentences using {\em three} non terminal symbols, namely~$\framebox{$0$} \rightarrow$, \framebox{$1$} and~$\framebox{$1$} \rightarrow$. To emphasize this interpretation, we can adopt more symmetric notations and write these three non terminal symbols as~$[_0$, $]_1$ and~$[_1$. With these new notations, the representation of our original sentence is now
\begin{align*}
  [_0 & \text{\it ~This is } ]_1 \text{\it ~Peter .}\\
  [_1 & \text{\it ~my friend}
\end{align*}
In this new representation, the rewriting rules can be replaced by merge operations of the type
\[
  a \, ]_i c + [_i \, b ~ \mapsto ~ a b c
\]
We can make this merge operations even more symmetric, if we consider that each expression can be represented by any of its circular permutations. Indeed, each expression contains exactly one non terminal symbol of the form~$[_i$, and therefore is uniquely defined by any of its circular permutations (since, due to this feature, we can define the permutation in which the opening bracket~$[_i$ comes first as the canonical form, and recover it from any other circular permutation).
Using this convention, we can write~$a\,]_i c$ as~$c a\,]_i$ and describe the merge operation as
\[
ca \, ]_i + [_i \, b ~ \mapsto ~ cab,
\]
or, renaming~$ca$ as~$a$ simply as
\[
  a ]_i + [_i b ~ \mapsto ~ ab.
\]

Let us formalize what we have explained so far. Let~$D$ be some dictionary of words (which can be for the sake of this mathematical description any finite set, representing the words of the natural language to be modeled). Let us form the symbol set~$S = D \cup \bigl\{ [_i, ]_i, i \in \B{N} \bigr\}$. Let us define the set of circular permutations of a sequence of symbols as 
\[
  \F{S}(w_0, \dots, w_{\ell-1}) = \bigl\{( w_{(i+j\!\!\!\!\mod \ell)} , i=0, \dots, \ell-1), j = 0 , \dots,  \ell-1 \bigr\},
\]
so that for instance $\F{S}(w_0,w_1,w_3) = \{ w_0 w_1 w_2, w_1 w_2 w_0, w_2 w_0 w_1 \}$,
and its support (the set of symbols included in the sequence) as
\[
  \supp \bigl( w_0, \dots, w_{\ell-1} \bigr) = \bigl\{ w_0, \dots, w_{\ell-1} \bigr\}.
\]
Let~$A^+ = \bigcup_{n=1}^{+\infty} A^n$, $\,]_+ = \bigl\{ \; ]_i, i \in \B{N} \setminus \{ 0 \} \bigr\}$, and consider the set of expressions 
\[
  \C{E} = \Bigl\{ \; e \in \F{S} \bigl([_i \, a \bigr) , \; i \in \B{N}, \; 
a \in  \bigl( D \, \cup \; ]_+ \bigr)^+ \; \setminus \; ]_+ \Bigr\}.
\]
In plain words, an expression is a circular permutation of a finite 
sequence of symbols starting with an opening bracket, containing 
no other opening bracket and not reduced to an opening bracket 
followed by a closing bracket.

This definition mirrors the fact that a given rule of a Context Free 
Grammar has exactly one~$\framebox{$i$} \rightarrow$ (the left side), and the right side of the rule cannot be just a non terminal symbol~$\framebox{$j$}$. 
Indeed, if we had allowed $\framebox{$i$} \rightarrow \framebox{$j$}$, 
or with our notations $[_i \, ]_j$, we could as well have replaced $i$ 
by $j$ everywhere.

\begin{dfn}\label{dfn:ToricGrammars}
  The set of toric grammars is the set~$\F{G}$ of positive measures~$\C{G}$ 
on $\C{E}$ with finite support such that for any circular permutation~$e' \in \F{S}(e)$ of any expression~$e \in \C{E}$, $\C{G}(e') = \C{G}(e)$. 
\end{dfn}

In other words, a toric grammar~$\C{G}$ is a  positive measure with finite support on the set of expressions~$\C{E}$ satisfying 
\[
  \C{G}(e) =\lvert \F{S}(e) \rvert^{-1} \C{G} \bigl( \F{S}(e) \bigr).
\]

Let us remark that, in our definition of toric grammars, on top of choosing some special notations for Context Free Grammars, we also introduced positive weights, so that it is more the support of a toric grammar than the grammar itself that corresponds to the usual notion of Context Free Grammar. 

The weights will serve to keep track of word frequencies through the process of splitting a set of sentences to obtain a toric grammar. 

Our aim is indeed to build a toric grammar from a text. To be consistent with our definition of grammars, we will also define texts as positive measures. Let us give a formal definition. We will forget the sentence order, a text will be an unordered set of sentences with possible repetitions.

\begin{dfn}\label{dfn:Texts}
  The set~$\F{T}$ of texts is the set of toric grammars with integer weights supported by $\F{S} \left( [_0 \, D^+ \right)$, that is the  set of 
toric grammars with integer weights using only one non terminal symbol, the start symbol~$[_0$.
\end{dfn}

In this definition, it should be understood that 
\[
  [_0 \, D^+ = \bigl\{ \bigl( [_0, w_1, \dots, w_{k} \bigr), \text{ where } 
  k \in \B{N} \setminus \{0\} \text{ and } w_i \in D, 1 \leq i \leq k \bigr\},
\]
and that 
\[
  \F{S} \bigl( [_0 \, D^+ \bigr) = \bigcup_{e \in [_0\,D^+} \F{S}(e).
\]

\section{A roadmap towards a communication model}

We will use toric grammars as intermediate steps to define the
transition probabilities of our communication model on texts.
To this purpose, we will first introduce some general types of transformations
on toric grammars (reminding the reader that in our formalism
 texts are some special subset of toric grammars).

It will turn out that two types of expressions, global expressions and
local expressions, will play different roles. Let us define them
respectively as
\begin{align*}\label{df*:localglobal}
\C{E}_{g} & = \C{E} \cap \F{S} \bigl( [_0 \, S^+ \bigr), \\
\C{E}_{\ell} & = \C{E} \cap \F{S} \bigl( [_+ \, S^+ \bigr), 
\end{align*}
where we remind that $\ds [_+ = \bigl\{ [_i, i \in \B{N} \setminus \{ 0 \} 
\bigr\}$ and $S^+ = \bigcup_{j=1}^{\infty} S^j$.
Any toric grammar $\C{G} \in \F{G}$ can be accordingly decomposed into
$\C{G} = \C{G}_g + \C{G}_{\ell}$, where $\C{G}_g (A) = \C{G} \bigl(
A \cap \C{E}_g \bigr)$ and $\C{G}_{\ell}(A) = \C{G} \bigl(
A \cap \C{E}_{\ell} \bigr)$, for any subset $A \subset \C{E}$.

The transitions of the communication chain with kernel 
$q_{\theta}(\C{T}, \C{T}')$
will be defined in two steps. The first step consists in learning 
from the text $\C{T}$ a toric grammar $\C{G}$. To this purpose
we will split the sentences of $\C{T}$ into syntactic constituents.
The second step consists in merging the constituents again to 
produce a random new text $\C{T}'$. 
The parameter $\theta = \C{R}$ of the communication kernel $q_{\theta}$, 
will also be a toric grammar. The role of this reference grammar 
$\C{R}$ will be to provide a stock of local expressions to 
be used when computing $\C{G}$ from $\C{T}$. We will discuss 
later the question of the estimation of $\C{R}$ itself. For 
the time being, we will assume that the reference grammar 
$\C{R}$ is a parameter of the communication chain, 
known to all involved speakers.  

We could have defined a communication kernel 
$q_{\wh{\C{R}}(\C{T})}(\C{T},\C{T}')$, where the reference 
grammar $\wh{\C{R}}(\C{T})$ itself is estimated at each step from the 
current text $\C{T}$,
but we would have obtained a model with weaker properties,
where, in particular, all the states are not necessarily 
recurrent states. On the other hand, the proof that the reachable 
set from any starting point is finite still holds for 
this modified model, so that it does provide an alternative 
way of defining a language model as described 
in the introduction.

We will still need an estimator $\wh{\C{R}}(\C{T})$ 
of the reference grammar, 
in order to provide a language estimator 
$\wh{Q}_{\wh{\C{R}}(\C{T}), \C{T}'}$, where we are using the 
notations of \vref{eq1.1}. The estimation $\wh{\C{R}}(\C{T})$ of the 
reference grammar will be achieved by running some 
fragmentation process on the text $\C{T} \in \F{T}$. 

\section{Non stochastic syntax splitting and merging}

Let us now describe the model, starting with the description of some
non random grammar transformations. We already introduced a model
for grammars that includes texts as a special case. We have
now to describe how to generate a toric grammar from a text,
with, or without, the help of a reference grammar to
learn the local component of the grammar. The mechanism
producing a grammar from a text will be some sort of random parse
algorithm (or rather tentative parse algorithm).

All of this will be achieved by two transformations on toric grammars that will
respectively {\em split} and {\em merge} expressions (syntagms)
of a toric grammar into smaller or bigger ones. We will first
describe the sets of possible splits and merges from a given grammar.
This will serve as a basis to define random transitions from one
grammar to another in subsequent sections.

Let us first introduce some elementary operations involving toric grammars.
\begin{align*}
  e \oplus f & = \sum_{s \in \F{S}(e)} \delta_s + \sum_{s \in \F{S}(f)} \delta_s, & e , f \in \C{E}, \\
  e \ominus f & = \sum_{s \in \F{S}(e)} \delta_s - \sum_{s \in \F{S}(f)} \delta_s, & e , f \in \C{E}, \\
  \rho \otimes e & = \rho \sum_{s \in \F{S}(e)} \delta_s, & \rho \in \B{R}, e \in \C{E}, 
\end{align*}
The first operation builds a toric grammar containing expressions~$e$ and~$f$ with weights~$1$, and the third one builds a toric grammar containing expression~$e$ with weight~$\rho$. 

We can generalize these notations to be able to take the sum of a toric grammar and an expression, as well as the sum of two toric grammars.
\begin{align*}
  \C{G} \oplus e & = \C{G} + \sum_{s \in \F{S}(e)} \delta_s, & \C{G} \in \F{G}, e \in \C{E}\\ 
  \C{G} \ominus e & = \C{G} - \sum_{s \in \F{S}(e)} \delta_s, & \C{G} \in \F{G}, e \in \C{E}\\ 
  \C{G} \oplus \C{G}' & = \C{G} + \C{G}',  & \C{G}, \C{G}' \in \F{G}.
\end{align*}
With these notations, a split is described as 
\[
  \C{G}' = \C{G} \ominus ab \oplus a ]_i \oplus [_i b, \qquad 
\C{G}, \C{G}' \in \F{G},
\]
the fact that $\C{G}, \C{G}' \in \F{G}$ implying that 
\[
i \in \B{N} \setminus \{ 0 \}, ab, a \, ]_i, [_i b \in \C{E} \text{ and }  \C{G}(ab) \geq 1.
\]
The (partial) order relation~$\C{G} \leq \C{G}'$ will also be defined by the rule 
\[
  \C{G} \leq \C{G}' \iff \C{G}' - \C{G} \in \F{G},   
\]
or equivalently
\[
  \C{G} \leq \C{G}' \iff \C{G}' - \C{G} \in \C{M}_+(\C{E}).   
\]
Let us resume our example. Starting from the one sentence text 
\[
\C{T} = 1 \otimes [_0 \text{\it ~This is my friend Peter .}
\]
we get after splitting the grammar
\[
  \C{G} = [_0 \text{\it ~This is } ]_1 \text{\it ~Peter .} \oplus [_1 \text{\it ~my friend }
\]
which can also be written as
\[
  \C{G} = \text{\it ~Peter . } [_0 \text{\it ~This is } ]_1 \oplus [_1 \text{\it ~my friend}
\]
In this example, as well as in the following, punctuation marks are treated as words, so that here the required dictionary has to include~$\{$ is, my, friend, Peter, This, . $\}$.

Splitting a sentence providing a new label for each split does not create generalization, 
since it allows only to merge back two expressions that came from the same split. To create a grammar capable of yielding new sentences, we need some label identification scheme. We will perform label identification through the more general process of label remapping, identification being a consequence of the fact that the map may not be one to one. Let
\[
  \F{F} = \bigl\{ f: \B{N}
  \rightarrow \B{N} \text{ such that }  f(0) = 0   \bigr\} 
\]
be the set of label maps. For any symbol~$]_i$ or~$[_i$, $i \in \B{N}$, let us define~$f(\, ]_i) = ]_{f(i)}$ and~$f ([_i \, ) = [_{f(i)}$. Let us also define for any word~$w \in D$, $f(w) = w$ and for any expression~$e = (w_0, \dots, w_{\ell -1})$, $f(e) = ( f(w_0), \dots, f(w_{\ell -1}) )$. Since any grammar~$\C{G} \in \F{G}$ is a measure on the set of expressions~$\C{E}$, we can define its image measure by~$f$, considered as a map from~$\C{E}$ to~$\C{E}$. We will put $f(\C{G}) = \C{G} \circ f^{-1}$, meaning that $f(\C{G})(A) = \C{G} ( f^{-1} (A) )$, for any subset~$A \subset \C{E}$.

  \begin{dfn}\label{dfn:Isomorph}

  Two label maps~$f$ and~$g \in \F{F}$ are said to be isomorphic if there is a one to one label map~$h \in \F{F}$ such that $g = h \circ f$. In this case $h^{-1} \in \F{F}$ and $f = h^{-1} \circ g$. Two grammars~$\C{G}$ and~$\C{G}' \in \F{G}$ are said to be isomorphic if there is a one to one label map~$f \in \F{F}$ such that~$f(\C{G}) = \C{G}'$. In this case, $f^{-1} ( \C{G}' ) = \C{G}$ and we will write $\C{G} \equiv \C{G}'$. If~$f$ and~$g$ are two isomorphic label maps, then for any toric grammar~$\C{G} \in \F{G}$, $f(\C{G})$ and~$g( \C{G} )$ are isomorphic grammars. In the following of this paper, to ease notations and simplify exposition, we will freely identify isomorphic label maps and isomorphic grammars and often speak of them as if they were equal.
\end{dfn}

This being put, we proceed with the introduction of a set of grammar transformations~$\beta$ that consist in a split with possible label remapping. The \emph{split} will be the core component for generating a toric grammar from a text, by splitting the sentences in smaller parts (syntagms).
\begin{dfn}[Splitting rule]\label{dfn:Split}

For any~$\C{G} \in \F{G}$, let us consider
  \[
    \beta ( \C{G} ) = \Bigl\{ f(\C{G}'), f \in \F{F}, 
    \C{G}' \in \F{G}, \C{G}'  = \C{G} \, \ominus \, ab \, \oplus \, a ]_i 
    \, \oplus \, [_i b 
    \; \Bigr\} \subset \F{G}.
  \]
  Let us remark that in this definition, necessarily, $ab, a \, ]_i, [_i \, b \in \C{E}$, $i \in \B{N} \setminus \{0\}$, $1 \otimes ab \leq \C{G}$, and $a \,]_i \oplus [_i \, b \leq \C{G}'$. Let us put
  \[
    \beta^* (\C{G}) = \bigcup_{n=0}^{+ \infty} \underbrace{\beta \circ \dots \circ \beta}_{n \text{ times}}(\C{G}),
  \]
  the set of grammars that can be constructed from repeated invocations of~$\beta$.
\end{dfn}

\begin{lemma}\label{lm:Conservation}

  Let us recall that $S = D \cup \bigl\{ \, [_i, \, ]_i, i \in \B{N} \bigr\}$ and  let us put $S^* = \bigcup_{n=0}^{+ \infty} S^n$. For any text~$\C{T} \in \F{T}$, and any~$\C{G} \in \beta^*(\C{T})$, $\C{G}$ is a toric grammar with integer weights, 
  \begin{align*}
    \C{G}([_i S^*) & = \C{G}(]_i S^*), & i \in \B{N} \setminus \{0\},\\
    \C{G}(w S^*) & = \C{T}(w S^*), & w \in  \bigl( D \cup \{ [_0 \, \} \bigr),\\  
    \intertext{and in particular}
    \C{G} ( [_0 S^*) & = \C{T}([_0 S^*),\\
    \C{G}(w S^*) & \leq \C{T}(w S^*), & w \in \bigl( D \cup \{[_0 \, \} \bigr)^+.  
  \end{align*}
  This means that in any toric grammar obtained by splitting a text, the weights of expressions containing the two forms~$]_i$ and~$[_i$ of a label are balanced, the word frequencies are the same in the grammar and in the text, and the number of sentences contained in the text is given by the total weight of expressions containing the start symbol~$[_0$ in the grammar. 
\end{lemma}

\begin{proof}
  For the first assertion, an induction on the number of applications of~$\beta$ yields the result, since
  \[
    \C{T}([_i S^*) = \C{T}(]_i S^*)=0, i \in \B{N} \setminus \{0\},
  \]
  and, for any~$\C{G}' = \C{G} \ominus 
ab \oplus a \, ]_i \oplus [_i \, b$, and any label $j \in \B{N} \setminus 
\{ 0 , i \}$,  
  \begin{gather}
    \C{G}'([_j S^*)=\C{G}([_j S^*),\\
    \C{G}'(]_j S^*)=\C{G}(]_j S^*), 
  \end{gather}
whereas
  \begin{gather}
    \C{G}'([_i S^*)=\C{G}([_i S^*)+1,\\
    \C{G}'(]_i S^*)=\C{G}(]_i S^*)+1.
  \end{gather}
  
  For the second assertion, it suffices to remark that the weight of 
expressions beginning with a given word is invariant by application of~$\beta$. 
Indeed, any word symbol $w \in D \cup \{ [_0 \}$ appears the same number 
of times at the beginning of an expression of $1 \otimes ab$ and of $a \, ]_i 
\oplus [_i \, b$. 
\end{proof}

This lemma is important, because we will subsequently impose restrictions on the splitting rule based on word frequencies. Our choice to define a new type of grammar as a positive measure on symbol sequences was made to keep track of word frequencies throughout the construction.

Let us now describe the reverse of a splitting transformation, that we will call a merge transformation. This transformation will be central in generating new texts from a toric grammar, by merging the syntagms into bigger ones, ending with a full sentence.

\begin{dfn}[Merge rule]\label{dfn:Merge}
  For any toric grammar~$\C{G} \in \F{G}$ we consider the following set of allowed merge transformations
  \[
    \alpha(\C{G}) = \Bigl\{ \; \C{G}' \in \F{G}, \C{G}' = \C{G} \, \ominus \, a \, ]_i \, \ominus \, [_i \, b \, \oplus \, ab \;  \Bigr\}.  
  \]
  Let us remark that in this definition, necessarily $i \in \B{N} \setminus \{0\}$, $a \, ]_i, [_i\, b, ab \in \C{E}$, and $a \, ]_i \oplus [_i \, b \leq \C{G}$.
\end{dfn}

The merge transformation is indeed the reverse of the \emph{split}, in the sense that:

\begin{lemma}\label{lm:Merge}
  For any~$\C{G}, \C{G}'\in \beta^*(\F{T})$, $\C{G}'\in\beta(\C{G})$ if, and only if, there is~$f \in \F{F}$ such that $f(\C{G}) \in \alpha(\C{G}')$.
\end{lemma}

\begin{proof}
  Let us suppose that $\C{G}' = f \bigl(\C{G} \oplus a \, ]_i \oplus [_i \, b \ominus ab \bigr)$ is in $\beta(\C{G})$. Then $\C{G}' = f\bigl( \C{G} \bigr) \oplus f(a \, ]_i) \oplus f([_i \, b) \ominus f(ab)$, so that $f(a \, ]_i), f([_i \, b) \in \supp\bigl(\C{G}'\bigr)$, $f(ab) \in \supp\bigl( f(\C{G}) \bigr)$, and consequently $f \bigl(a \, ]_i \bigr), f \bigl( [_i b \bigr)$ and $ f(ab) \in \C{E}$. Moreover $f(\C{G}) = \C{G}' \oplus f(a) f(b) \ominus f(a) \, ]_{f(i)} \ominus [_{f(i)} \, f(b)$, so that $f(\C{G}) \in \alpha(\C{G}')$.

  On the other hand, if for some~$f \in \F{F}$, $f(\C{G}) \in \alpha \bigl( \C{G}' \bigr)$, $f(\C{G}) = \C{G}' \oplus a b \ominus a \, ]_i \ominus [_i \, b$. Since $ab \in \supp \bigl( f(\C{G}) \bigr)$, there is~$e \in \C{E}$ such that~$f(e) = ab$. But this implies that there is~$c, d \in S^+$ such that~$a = f(c)$ and~$b = f(d)$. We can then if needed modify~$f$ outside $\bigl\{ j \in \B{N}: [_j \, S^* \in \supp(\C{G}) \bigr\}$, to make sure that $i \in f(\B{N})$. Let~$f(j) = i$. We now get that $f\bigl(\C{G}\bigr) = \C{G}' \oplus f(c) f(d) \ominus f(c) \, ]_{f(j)} \ominus [_{f(j)} \, f(d)$, so that $\C{G}' = f \bigl( \C{G} \oplus c \, ]_j \oplus [_j \, d \ominus cd \bigr)$, proving that $\C{G}' \in \beta(\C{G})$.
\end{proof}

\medskip
Another useful property of the merge rule is given by the following lemma:
\begin{lemma}\label{lm:MergeLabel}
  For any~$f \in \F{F}$ and any~$\C{G} \in \F{G}$, $f \bigl( \alpha( \C{G} ) \bigr) \subset \alpha \bigl( f(\C{G}) \bigr)$.
\end{lemma}

\begin{proof}
  Indeed, any~$\C{G}' \in f \bigl( \alpha( \C{G} ) \bigr)$ is of the form
  \begin{align*}
    \C{G}' & = f \bigl( \C{G} \oplus ab \ominus a \, ]_i \ominus [_i \, b \bigr) \\ 
    & = f \bigl( \C{G} \bigr) \oplus f(a) f(b) \ominus f(a) \, ]_{f(i)} \ominus [_{f(i)} \, b \in \alpha \bigl( f(\C{G}) \bigr).
  \end{align*}
\end{proof}

\medskip
Unfortunately, repeating the merge transformation will not provide a text in all circumstances. Indeed, we can end up with some expressions of the type~$[_i \, a \, ]_i b$. However, since an expression is allowed to contain only one opening bracket, we are sure that $[_0 \, \not\in \supp( [_i \, a \, ]_i \, b )$.

To continue the discussion, we will switch to a random context, where split and merge transformations are performed according to some probability measure.

\section{Random split and merge processes}

The grammars we described so far are obtained using splitting rules. Texts can be reconstructed using merge transformations. The splitting rules as well as the merge rules allow for multiple choices at each step. We will account for this by introducing random processes where these choices are made at random.

We will describe two types of random grammar transformations. Each of these will appear as a finite length Markov chain, where the length of the chain is given by a uniformly bounded stopping time.  
\begin{itemize}
  \item The learning process (or splitting process) will start with a text and build a grammar through iterated splits;
  \item the production process will start with a grammar and produce a text through iterated merge operations.
\end{itemize}
These two types of processes may be combined into a split and merge process, going back and forth between texts and toric grammars.

Let us give more formal definitions. Learning and parsing processes will be some special kinds of splitting processes, to be defined hereafter.

\begin{dfn}[Splitting process]\label{dfn:SplitProc}
  Given some restricted splitting rule~$\beta_r: \F{G} \rightarrow 2^{\F{G}}$ from the set of grammars to the set of subsets of~$\F{G}$, such that for any~$\C{G} \in \F{G}$, $\beta_r(\C{G}) \subset \beta(\C{G})$, a \emph{splitting process} is a time homogeneous stopped Markov chain~$S_t, 0 \leq t \leq \tau$ defined on~$\F{G}$
such that 
  \[
    \tau = \inf \bigl\{ t \in \B{N}: \beta_r(S_t) = \varnothing \bigr\},
  \]
\[
\B{P} \bigl( S_t = \C{G}'\,|\,S_{t-1} = \C{G} \bigr) > 0 \iff \C{G}' 
\in \beta_r(\C{G}).
\]
\end{dfn}

\begin{dfn}[Production process]\label{dfn:ProdProc}

  A \emph{production process} is a time homogenous stopped Markov chain~$P_t, 
0 \leq t \leq \sigma$ defined on~$\F{G}$ such that 
  \[
    \sigma = \inf \bigl\{ t \in \B{N}, \alpha(P_t) = \varnothing \bigr\},
  \]
and 
\[
\B{P} \bigl( P_t = \C{G}' \, | \, P_{t-1} = \C{G} \bigr) > 0 \iff \C{G}' \in 
\alpha(\C{G}).
\]
\end{dfn}

\begin{dfn}[Split and Merge process]\label{dfn:SplitMergeProc}

  Given a splitting process~$S_t, t \in \B{N}$ and a production process~$P_t, t \in \B{N}$, a \emph{split and merge process} is a Markov chain~$G_t \in \F{G}$, $t \in \B{N}$, with transitions
  \begin{align*}
    \B{P} \bigl( G_{2t+1} = \C{G}' \, | \, G_{2t} = \C{G} \bigr) & =
    \B{P} \bigl( S_{\tau} = \C{G}' \, | \, S_0 = \C{G} \bigr), & t \in \B{N},\\
    \B{P} \bigl( G_{2t} = \C{G}' \, | \, G_{2t-1} = \C{G} \bigr) & = 
    \B{P} \bigl( P_{\sigma} = \C{G}' \, | \, P_0 = \C{G}, P_{\sigma} \in \F{T} 
    \bigr), & t \in \B{N} \setminus \{ 0 \},
  \end{align*}
  whose initial distribution is a probability measure on 
texts, so that almost surely $G_0 \in \F{T}$.
\end{dfn}
Let us remark that we have to impose the condition that~$P_\sigma \in \F{T}$, because the production process does not produce a true text with probability one. On the other hand it can yield back~$G_{2t -2}$ with positive probability 
when started at~$G_{2t - 1}$, as will be proved later on.
Therefore $\B{P} ( P_{\sigma} \in \F{T} \, | \, P_0 = \C{G} ) > 0 $ for 
any~$\C{G}$ such that $\B{P}(G_{2t-1} = \C{G} ) > 0$.
One way to simulate $\B{P}_{G_{2t} \, | \, G_{2t-1}}$ is to use a rejection method, simulating repeatedly from the production process until a true text is produced. In the experiments we made, $\B{P} \bigl( P_{\sigma} \in \F{T} 
\, | \, P_0 = \C{G} \bigr)$ was close to one and rejection a rare event. 

\begin{prop}\label{pp:BoundSMProc}

Let~$S_t$, $P_t$ and~$G_t$ be a splitting process, a production process and 
the corresponding split and merge process, starting from 
$G_0 = \C{T} \in \F{T}$. 
For any $\C{G} \in \F{G}$, any $\C{T}' \in \F{T}$, such that 
$\sum_{t \in \B{N}} \B{P}(G_{2t+1} = \C{G}) > 0$ and $
\sum_{t \in \B{N}} \B{P} ( G_{2t} = \C{T}') > 0$,
\begin{gather}
  \B{P} \Bigl( \tau \leq 2 \bigl[ \C{T} \bigl( D S^* \bigr) - \C{T} 
\bigl( [_0 \, S^* \bigr) \bigr] \, \bigl| \, S_0 = \C{T}' \Bigr)  = 1,\\
  \B{P} \Bigl( \sigma \leq 2 \bigl[ \C{T} \bigl( D S^* \bigr) - 
\C{T} \bigl( [_0 \, S^* \bigr) \bigr] \, \bigl| \, P_0 = \C{G} \Bigr) = 1.
\end{gather}
In other words, the length of all the splitting and production processes involved in the split and merge process have a uniform bound, given by twice the difference between the number of words and the number of sentences in the original text.
\end{prop}

\begin{proof}
  This proof is a bit lengthy and is based on some invariants in the split and merge operations. It has been put off to \vref{App:BoundSMProc}.
\end{proof}

\begin{prop}\label{pp:FiniteState}

  If~$G_t$ is a split and merge process starting almost surely from the text~$G_0 = \C{T} \in \F{T}$, there is a finite subset of toric grammars~$\F{G}_{\C{T}}$ such that with probability equal to one there is for each time~$t$ a grammar~$G'_t$ isomorphic to~$G_t$ such that $G'_t \in \F{G}_{\C{T}}$. Thus, after identification of isomorphic grammars, we can analyze the split and merge process as a finite state Markov chain, since the reachable set from
any starting point is finite. We should however keep in mind that the finite state space~$\F{G}_{\C{T}}$ depends on the initial state~$\C{T}$, so the state space is still infinite, although any trajectory will almost surely stay in a finite subset of reachable states.
\end{prop}

\begin{proof}
Let us assume that the labels of~$\C{G}$ are taken from~$\llbracket 0, 
W_{\ell}(\C{G}) \rrbracket$, meaning that 
$\C{G} ( [_i \, S^* ) = 0$ for~$i > W_{\ell}(\C{G})$. 
This can be achieved, up to 
grammar isomorphisms, by applying to~$\C{G}$ a suitable label map.\\
Let us define the set of canonical expressions 
$$
\C{E}_c = \C{E} \cap \Biggl( \; \bigcup_{i \in \B{N}} [_i \, S^* \Biggr),
$$
and the canonical decomposition of~$\C{G}$ 
\[
\C{G} = \sum_{e \in \C{E}_c} \C{G}(e) \otimes e.
\]
We see that $\C{G}$ can be described by the concatenation of the 
canonical expressions, each repeated a number of times equal 
to its weight, to form a sequence of symbols of length $W_s(\C{G})$.
From the proof of the previous proposition, we know that
\[
W_s(\C{G}) \leq M = 5 W_{w}(\C{T}) - 3 W_e(\C{T}) = 5 \C{T}(DS^*) - 
3 \C{T}([_0 S^*).
\]
We can represent $\C{G}$ by a sequence of exactly $M$ symbols by 
padding with trailing $[_0$ symbols the representation described above.  
Let us give an example
\[
\C{G} = 2 \otimes [_0 \, w_1 \, ]_1 w_2 \oplus [_1 \, w_3 \oplus [_1 
\, w_4
\]
can be coded as
\[
[_0 \, w_1 \, ]_1 w_2 \, [_0 \, w_1 \, ]_1 w_2 \, [_1 \, w_3 \, [_1 \, w_4 
\, [_0 \, [_0 \, [_0
\]
in the case when $M = 15$. 
Let us consider the set of symbols 
\[
S_{\C{T}} = D \cup \bigl\{ [_0 \, , [_i \, , \, ]_i, 0 < i \leq 
2 \bigl[ \C{T}(DS^*) - \C{T}([_0 \, S^*) \bigl] \bigr\}.
\]
Since $\C{G}$ uses only those symbols, we see from the proposed coding 
of $\C{G}$ that it can take at most 
\[
\lvert S_{\C{T}} \rvert^{M}
\] 
different values.
Since 
\[
\lvert S_{\C{T}}  \rvert = 
\lvert D \rvert + 1 + 4 \bigl[ \C{T}(D S^*) - \C{T}([_0 S^*) \bigr] 
\]
we have proved that 
\[
\lvert \F{G}_{\C{T}} \rvert \leq 
\Bigl( \lvert D \rvert + 1 + 4 \bigl[ \C{T}(D S^*) - \C{T}([_0 S^*) \bigr]
\Bigr)^{\ds 5 \C{T}(D S^*) - 3 \C{T} ( [_0 \, S^*)}.
\]
Let us notice that this bound, while being finite, is very large. 
\end{proof}

\section{Splitting rules and label identification}

In the previous section, we introduced some class of random processes, and studied some of their general properties. In this section, we are going to describe some more specific schemes and go further in the description of split and merge processes that can learn toric grammars in a satisfactory way.

The choice of splitting rules and label identification rules has a decisive influence on the way syntactic categories and syntactic rules are learnt by the split and merge process. While it is necessary as a starting point to consider rules learnt from the text to be parsed itself, it will also be fruitful to consider the case when a previously learnt grammar~$\C{R} \in \F{G}$ can be used to govern the splits.

To make things easier to grasp, let us explain on some example the basics of syntactic generalization by label identification. Let us start with the simple text with two sentences.
\[
  G_0 = \C{T} = [_0 \, \text{ This is my friend Peter . } \oplus [_0 \, \text{ This is my neighbour John .}
\]
If we split ``my friend'' and ``my neighbour'' in the two sentences using the same label, we will form after two splits the grammar
\begin{align*}
  G_1 & = [_0 \, \text{ This is } \, ]_1 \text{ Peter .} 
  \; \oplus \; [_0 \, \text{ This is } ]_1 \text{ John .} 
  \\ & \; \oplus \; [_1 \, \text{ my friend } \; \oplus \; [_1 \, \text{ my neighbour }
\end{align*}
If no more splits are allowed and we therefore reached the stopping time of the splitting process, so that $\tau = 2$, we can proceed to the production process, and reach after two more steps the new text~$G_2$ that can either be $G_2 = G_0$ or
\[
  G_2 = [_0 \, \text{ This is my neighbour Peter . } \oplus [_0 \, \text{ This is my friend John . }
\]

Now is a good time to remind the reader of the distinction made in \vref{df*:localglobal} about local and global expressions.

Legitimate local expressions will be provided by the reference grammar 
$\C{R}$, whereas global expressions will be deduced from the text 
itself. This approach will be particularly efficient in the case
when the set of local expressions is smaller than the set of global 
expressions. 

We will need two different kinds of split processes, one 
to learn the reference grammar from a text and the other one 
to perform the first part of the transitions of the 
communication Markov chain.  

These split processes may be viewed 
as performing some parsing of the text they are applied to. 
Here, we do not use parsing as it is usually used to discover 
whether a sentence is correct or not, we use it instead to 
discover new expressions. 

We will start by defining the parsing rules to be used in 
the communication chain.
We will call them \emph{narrow} parsing rules.  
We will then proceed to the definition 
of a \emph{broad} parsing rule suitable for learning 
the reference grammar $\wh{\C{R}}(\C{T})$ from a text.

\begin{dfn}\label{dfn:NarrowParse}

  Let us define the narrow parsing rule with reference grammar~$\C{R}$  as
  \begin{align*}
    \beta_{n} \bigl( \C{G}, \C{R} \bigr) = \Bigl\{ & \; \C{G}' \in \F{G}: \C{G}' = \C{G} \, \oplus a \, ]_i \, \oplus \, [_i \, b \, \ominus \, ab, \\ & \; ab \in \C{E}_g, \; \C{R} \bigl( [_i \, b \bigr) > 0 \; \Bigr\}, \quad \C{G} \in \F{G}.
  \end{align*}
  Let us remark that, due to the definition of the set of expressions~$\C{E}$ and of~$\F{G} \subset \C{M}_+(\C{E})$, the fact that $\C{G}$ and $\C{G'} \in \F{G}$ implies that $i \in \B{N} \setminus \{0\}$ in this definition, since necessarily $a \, ]_i , [_i \, b \in \C{E}$. It implies also that $[_0 \in \supp(a)$, a condition equivalent to $ab \in \C{E}_g$.
  
  The narrow parsing rule depends on $\C{R}$ only through $\supp(\C{R}) \cap \C{E}_l$. 

Let us define the broad parsing rule as 
  \begin{align*}
    \beta_b \bigl(\C{G}, \C{R} \bigr)  =  \Bigl\{ \; & \C{G}' \in \F{G}: \C{G}' = \C{G} \oplus a \, ]_i \oplus [_i \, b \ominus a b, \\
    & \C{R}\bigl( a \, ]_i \bigr) + \C{R} \bigl( [_i \, b 
\bigr) > 0, 
 \C{R} \bigl( a S^* \bigr) \leq \mu_1 \C{R} \bigl( [_0 \, S^* \bigr), \\
& \text{ and } \C{R} \bigl( b S^* \bigr) \leq \mu_2 \C{R} \bigl( [_0 \, 
S^* \bigr) 
\; \Bigr\}, \qquad \C{G}, \C{R} \in \F{G}, 
  \end{align*}
where $\mu_1, \mu_2 \in \B{R}_+$ are two positive real parameters.
\end{dfn}
Since the reference grammar is under construction 
during broad parsing,  
we will mainly use this rule with $\C{R} = \C{G}$, 
as will be explained later.
The same learning parameters $\mu_1$ and $\mu_2$ are present here 
and in the innovation rule to be described next. 
They serve to split expressions into
sufficiently infrequent halves, in order to constrain the 
model. 

Let us define now maximal sequences, a notion that will be needed to define learning rules. 
\begin{dfn}\label{dfn:MaxSeq}
  Given some toric grammar~$\C{G}$, we will say that $a \in 
S^+$ is $\C{G}$-maximal and write $\ds a \in \max(\C{G})$ when 
  \[
    \C{G}(a S^*) > \max \bigl\{ \C{G}(awS^*), \C{G}(waS^*), w \in S \bigr\}. 
  \]
  In other words, $a$ is a maximal subsequence among the subsequences with the same weight in~$\C{G}$. Note that if $a$ is $\C{G}$-maximal, usually $\C{G}(a) = 0$ (meaning that $a$ is not an expression of the grammar, but only a subexpression) and if the grammar~$\C{G}$ has integer weights (which will be the case if it has been produced by a split and merge process), then $\C{G}(aS^*) \geq 2$.
\end{dfn}

\begin{dfn}[Innovation rule]\label{dfn:Innovation}

  Using the notations $[_+ \, = \bigl\{ [_i \, , i \in \B{N}   \setminus \{ 0 \} \bigr\}$ and $\, ]_+ = \bigl\{ \, ]_i, i \in \B{N} \setminus \{ 0 \} \bigr\}$, let us define the innovation rule with reference grammar~$\C{R}$ as
  \begin{align*}
  \beta_{i} \bigl( \C{G}, \C{R} \bigr)  = \Bigl\{ \; & \C{G}' \in \F{G}: \C{G}' = \C{G} \, \oplus \, a \, ]_i \, \oplus \, [_i \, b \, \ominus \, ab, \\
  & \C{R} \bigl( [_i S^* \bigr) = 0, \; \{a, b \} \cap \max(\C{R}) \neq \varnothing, \\
  & \C{R}(a S^*) \leq \mu_1 \, \C{R}\bigl( [_0 \, S^* \bigr), \text{ and } \; \C{R} \bigl( b S^* \bigr) \leq \mu_2 \, \C{R} \bigl( [_0 \, S^* \bigr) \; \Bigr\}.  
  \end{align*}
\end{dfn}
Here again, the rule will be used while learning the reference grammar
with $\C{R} = \C{G}$. 

We will now introduce a label map that identifies the labels 
appearing in the same context. 
\begin{dfn}[Label identification through context]\label{dfn:LabelIdentification}
Given some toric grammar 
$\C{G} \in \F{G}$, let us consider the relation 
$C \in \bigl( \B{N} \setminus \{ 0 \} \bigr)^2$ defined as
\[
C = \biggl\{ (i,j) \in \bigl( \B{N} \setminus \{ 0 \} \bigr)^2:
\sum_{a \in S^*} \C{G}(a \, ]_i ) \, \C{G} (a \, ]_j) 
+ \C{G}([_i \, a ) \, \C{G}([_j \, a) > 0 \biggr\}.
\]
The smallest equivalence relation containing 
$C$ defines a partition of $\B{N} \setminus \{ 0 \}$ 
into equivalence classes. 
Let $(A_k)_{k \in \B{N} \setminus \{ 0 \}}$ be an arbitrary indexing 
of this partition. Each positive integer falls in a unique class of the 
partition, so that the relation $i \in A_{\underline{\chi}_{\C{G}}(i)}$
defines a label map $\underline{\chi}_{\C{G}}: \B{N} \rightarrow \B{N}$
in a non ambiguous way. 
The choice of the indexing of the partition $(A_k)_{k \in \B{N} \setminus 
\{ 0 \}}$ does not 
matter, since two different choices  
lead to two isomorphic label maps. 
When applying $\und{\chi}_{\C{G}}$ to $\C{G}$ itself, we will use 
the short notation $\und{\chi}(\C{G}) \overset{\text{\rm def}}{=} 
\und{\chi}_{\C{G}}(\C{G})$.\\
Let us consider the evolution of the number 
of labels used by $\C{G}$:
\[
L(\C{G}) = \bigl\lvert \{ i \in \B{N}: \C{G}(\, ]_iS^*) > 0 \} \bigr\rvert.
\]
It is easy to see that $L \bigl( \underline{\chi}(\C{G}) \bigr) 
\leq L \bigl( \C{G} \bigr)$ and that $\underline{\chi}(\C{G}) \equiv 
\C{G}$ if and only if $L \bigl( \underline{\chi}_{\C{G}}(\C{G}) \bigr)
= L \bigl( \C{G} \bigr)$, where the symbol $\equiv$ means 
isomorphic. Accordingly there is 
$k \in \B{N}$ such that $\und{\chi}^{k+1} \bigl( \C{G} \bigr)  \equiv 
\und{\chi}^{k} \bigl( \C{G} \bigr) $, and we can take it 
to be the smallest 
integer such that $L \bigl( \und{\chi}^{k+1} ( \C{G} ) \bigr)
= L \bigl( \und{\chi}^k ( \C{G} ) \bigr)$.
Consequently, $k$ is such that 
for any $n \geq k$, $\und{\chi}^n \bigl( \C{G} \bigr) \equiv \und{\chi}^k 
\bigl( \C{G} \bigr)$. We will define $\chi(\C{G}) = \und{\chi}^k \bigl( \C{G} 
\bigr)$, up to grammar isomorphisms (so that $\chi(\C{G})$ 
belongs to $\quotient{\F{G}}{\equiv}$ rather than to $\F{G}$ itself).  
\end{dfn}
A characterisation in terms of more elementary label maps will 
be established in \vref{pp:ChiXi}. This characterization provides 
an algorithm to compute $\chi$ in practice.  

We are now ready to define a learning rule.
\begin{dfn}\label{dfn:Learning}
Let us define the learning rule 
\[
  \beta_{\ell}(\C{G}) = \begin{cases}
    \beta_i \bigl(\C{G}, \C{G} \bigr), & \text{when } \beta_b \bigl( \C{G}, \C{G} \bigr) = \varnothing, \\ 
    \bigl\{ \chi \bigl(\C{G}'\bigr): \C{G}' \in \beta_b \bigl( \C{G}, \C{G} \bigr) \bigr\}, & \text{otherwise}.
  \end{cases}
\]
\end{dfn}

We will define two kinds of splitting processes, based on two 
different choices of the restricted splitting rule~$\beta_r$. 

\begin{dfn}[Learning process]\label{dfn:LearningProc}
A \emph{learning process} is a splitting process with restricted splitting rule 
\[
  \beta_r(\C{G}) =  \beta_{\ell}(\C{G}).
\]
\end{dfn}

\begin{dfn}[Parsing process]\label{dfn:ParsingProc}

A \emph{parsing process} with reference grammar~$\C{R} \in \F{G}$ is a splitting process with restricted splitting rule 
\[
  \beta_r(\C{G}) = \beta_n( \C{G}, \C{R}).
\]
\end{dfn}

Before we reach the aim of this paper and describe our statistical language 
model, we need to explore some of the properties of the production, learning 
and parsing processes introduced so far.

\section{Parsing and generalization}
\label{section7}

Let us introduce some notations for the output of parsing, learning and production processes.
\begin{dfn}\label{dfn:Notations}

  Let $S_t$ be a parsing process, with reference grammar~$\C{R} \in \F{G}$. We will use the following notation for the distribution of~$S_{\tau}$.
  \[
    \B{G}_{\C{T}, \C{R}} = \B{P}_{S_{\tau} \, | \, S_0 = \C{T}}, \qquad \C{T} \in \F{T}.
  \]
  We will also use a short notation for the distribution of the output of a production process.
  \[
    \B{T}_{\C{G}} = \B{P}_{P_{\sigma} \, | \, P_0 = \C{G}, P_{\sigma} \in \F{T}}, \qquad \C{G} \in \F{G}.
  \]
  Eventually, $\B{G}_{\C{T}}$ will be the probability distribution of the output of a learning process~$S_t$, according to the definition 
  \[
    \B{G}_{\C{T}} = \B{P}_{S_{\tau} \, | \, S_0 = \C{T}}, \qquad \C{T} \in \F{T}.
  \]
\end{dfn}

At this point we obviously may consider different notions of parsing that we have to connect together. Namely, we would like to make a link between the following statements:
\begin{itemize}
  \item $\B{T}_{\C{G}}(\C{T}) > 0$, the grammar~$\C{G}$ can produce the text~$\C{T}$; 
  \item $\B{G}_{\C{T}, \C{R}}(\C{G}) > 0$, the text~$\C{T}$ can generate the grammar~$\C{G}$  when parsed with the help of the grammar~$\C{R}$;
  \item $\B{G}_{\C{T}}(\C{G}) > 0$, the grammar~$\C{R}$ can be learnt from the text~$\C{T}$.
\end{itemize}

\begin{lemma}\label{lm:ParseRelations}

  The previous parse notions are related in the following way. 
For any $\C{G}, \C{R} \in \F{G}$, and any $\C{T} \in \F{T}$,
  \begin{align*}
    \B{G}_{\C{T}}(\C{G})  > 0 & \implies \B{T}_{\C{G}}(\C{T}) > 0,\\
    \B{G}_{\C{T}, \C{R}}(\C{G}) > 0 & \implies \B{T}_{\C{G}}(\C{T}) > 0, \\ 
    \B{T}_{\C{G}}(\C{T}) > 0 & \implies \B{G}_{\C{T}, \C{G}}(\C{G}) > 0.
  \end{align*}
Consequently, 
for any $\C{G}, \C{R} \in \F{G}$ such that $\bigl( \supp(\C{G}) \cap 
\C{E}_l \bigr) \subset \supp (\C{R})$, and any $\C{T} \in \F{T}$,  
  \[
    \B{T}_{\C{G}}(\C{T}) > 0 \iff \B{G}_{\C{T}, \C{R}} (\C{G}) > 0.
  \]
\end{lemma}

\begin{proof}
  This is one of the core lemmas of this work. The proof is given in \vref{App:ParseRelations}, on account of its length.
\end{proof}

It has the following important implication.

\begin{prop}\label{pp:Markov}
Given a parsing process~$S_t$ based on a reference grammar~$\C{R} \in \F{G}$ and a production process~$P_t$, the corresponding split and merge process~$G_t$ is weakly reversible, in the sense that for any $\C{T} \in \F{T}$, any 
$\C{G} \in \bigcup_{t \in \B{N}} \supp \bigl( \B{P}_{G_{2t+1}}\bigr)$, 
$$
  \B{P} \bigl( G_1 = \C{G} \, \bigl| \, G_0 = \C{T} \bigr) > 0
    \iff \B{P} \bigl( G_2 = \C{T} \, \bigl| \, G_1 = \C{G} \bigr) > 0.
$$
  Consequently, for any~$\C{T}, \C{T}' \in \F{T}$ and any~$\C{G}, \C{G}' \in \bigcup_{t \in \B{N}} \supp\bigl( \B{P}_{G_{2t+1}}\bigr)$, 
  \begin{align*}
    \B{P} \bigl( G_2 = \C{T}' \, \bigl| \, G_0 = \C{T} \bigr) > 0
    & \iff 
    \B{P} \bigl( G_2 = \C{T} \, \bigl| \, G_0 = \C{T}' \bigr) > 0, \\
    \B{P} \bigl( G_3 = \C{G}' \, \bigl| \, G_1 = \C{G} \bigr) > 0 
    & \iff 
    \B{P} \bigl( G_3 = \C{G} \, \bigl| \, G_1 = \C{G}' \bigr) > 0. 
  \end{align*} 
  In other words, the two processes~$G_{2t}$ and~$G_{2t+1}$ are weakly reversible time homogeneous Markov chains. As we already proved that the set 
of reachable states from any starting point is finite, 
it shows that they are recurrent Markov chains: they
partition their respective state spaces into 
positive recurrent communicating classes. 
\end{prop}

\begin{proof}
  Let us remark first that 
  \begin{align*}
    \B{P} \bigl( G_1 = \C{G} \, \bigl| \, G_0 = \C{T} \bigr) > 0 & \iff  \B{G}_{\C{T}, \C{R}} \bigr( \C{G} \bigr) > 0	\\
    \B{P} \bigl( G_2 = \C{T} \, \bigl| \, G_1 = \C{G} \bigr) > 0 & \iff \B{T}_{\C{G}}\bigl( \C{T} \bigr) > 0.  
  \end{align*}
  Moreover, since $\C{G} \in \supp \bigl( \B{P}_{G_{2t+1}} \bigr)$ for some~$t \in \B{N}$, there is $\C{T}' \in \F{T}$ such that $\B{G}_{\C{T}', \C{R}}(\C{G}) > 0$, implying that $\supp \bigl( \C{G} \bigr) \cap \C{E}_l \subset \supp \bigl( \C{R} \bigr)$. 
  This ends the proof according to the last statement of the previous lemma.  
\end{proof}

\section{Expectation of a random toric grammar}
\label{sec:Expec}

In \vref{section7}, given some text $\C{T} \in \F{T}$, 
we defined a random distribution on toric grammars $\B{G}_{\C{T}}$
that we would like to use to learn a grammar from a text. 
The most obvious way to do this is to draw a toric grammar 
at random according to the distribution $\B{G}_{\C{T}}$, 
and we already saw an algorithm, described by a Markov chain 
and a stopping time, to do this.

The distribution $\B{G}_{\C{T}}$ will be spread in general 
on many grammars. This is a kind of instability that we would like
to avoid, if possible. A natural way to get rid of this instability 
would be to simulate the expectation of $\B{G}_{\C{T}}$.  
To do this, we are facing a problem: the usual definition
of the expectation of $\B{G}_{\C{T}}$, that is 
\[
\int \C{G} \, \ud \B{G}_{\C{T}} ( \C{G} ),
\]
although well defined from a mathemacial point of view, 
is a meaningless toric grammar, due to the possible 
fluctuations of the label mapping. To get a meaningful 
notion of  expectation, we need to define in a meaningful way the 
sum of two toric grammars. We will achieve this in two 
steps.   

Let us introduce first the {\em disjoint sum} of two toric 
grammars. We will do this with the help of two disjoint 
label maps. Let us define the {\em even} and {\em odd} 
label maps $f_e$ and $f_o$ as 
\[
f_e(i) = 2 i, \qquad f_o(i) = \max \{ 0, 2i-1 \}, \qquad i \in \B{N}.
\]
\begin{dfn}
The {\em disjoint sum} of two toric grammars $\C{G}, 
\C{G}' \in \F{G}$ is defined as 
\[
\C{G} \boxplus \C{G}' = f_e(\C{G}) + f_o(\C{G}'). 
\]
\end{dfn}

\begin{dfn}
Given a probability measure $\B{G} \in \C{M}_+^1(\F{G})$ with finite 
support, we define the mean of $\B{G}$ as
\[
\oint \C{G} \, \ud \B{G}(\C{G}) = \chi \Biggl( \bigboxplus_{
\C{G} \in \F{G}} \B{G}(\C{G}) \; \C{G} \Biggr).
\]
\end{dfn}

\begin{lemma}
\label{lm:Average}
If $G_i$ is an i.i.d. sequence of random grammars distributed 
according to $\B{G}$, then almost surely 
\[
\lim_{n \rightarrow + \infty} \frac{1}{n} \chi \Biggl( 
\bigboxplus_{i=1}^n G_i \Biggr) = \oint \C{G} \, \ud \B{G}(\C{G}). 
\]
\end{lemma}

\begin{proof}
  The proof of this result is quite lengthy, and postponed 
till~\vref{App:Average}.
\end{proof}

\section{Language models}

We are now ready to define the language model announced 
in the introduction. 
Given a reference grammar $\C{R}$, and the corresponding 
split and merge process $(G_t)_{t \in \B{N}}$ with reference
$\C{R}$, we define the communication kernel $q_{\C{R}}(\C{T}, \C{T}')$ 
on $\F{T}^2$ as 
$$
q_{\C{R}}(\C{T}, \C{T}') = \B{P} \bigl( G_2 = \C{T}' \, | \, G_0 = \C{T} 
\bigr).
$$
According to \vref{pp:FiniteState} and \vref{pp:Markov}, $q_{\C{R}}$ 
has finite 
reachable sets and is weakly reversible, so that all texts $\C{T} \in \F{T}$ are positive 
recurrent states of the communication kernel $q_{\C{R}}$. 

Thus to each text $\C{T} \in \F{T}$
corresponds a unique invariant text distribution 
$\wh{q}_{\C{R}}\bigl(\C{T}, \cdot \bigr)$, as explained in the 
introduction. As all states are positive recurrent, $\wh{q}_{\C{R}}(\C{T}, \cdot)$
is the unique invariant measure of $q_{\C{R}}$ on the communicating 
class containing $\C{T}$. Moreover, from the ergodic theorem,
$$  
\B{P} \Biggl( \wh{q}_{\C{R}}\bigl(\C{T}, \cdot \bigr) = \lim_{t \rightarrow 
\infty} \frac{1}{t} \sum_{j = 1}^t \delta_{G_{2j}} ~ \biggl\lvert ~ G_0 = \C{T} 
\Biggr) = 1,
$$
showing that $\wh{q}_{\C{R}}\bigl(\C{T}, \cdot \bigr)$ can be computed
by an almost surely convergent Monte-Carlo simulation. Eventually, 
from the invariant probability measure on texts $\wh{q}_{\C{R}} \bigl( 
\C{T}, \cdot \bigr)$, we deduce a probability measure on sentences 
$\wh{Q}_{\C{R}, \C{T}}$ as explained in the introduction, according 
to the formula
$$
\wh{Q}_{\C{R}, \C{T}} = \C{T} \bigl( [_0 S^* \bigr)^{-1} \sum_{\C{T}' \in \F{T}} 
\wh{q}_{\C{R}} \bigl( \C{T}, \C{T}' \bigr) \C{T}'.
$$
(This is the same formula as in the introduction, taking into account 
the fact that texts in the support of $\wh{q}_{\C{R}} 
\bigl( \C{T}, \cdot \bigr)$ are non normalized empirical measures 
with the same total mass equal to $\C{T} \bigl( [_0 S^* \bigr)$, 
the number of sentences in the text $\C{T}$.)

To obtain a true language estimator, there remains to estimate 
$\C{R}$ by some estimator $\wh{\C{R}}(\C{T})$. We will do this 
as described in \vref{sec:Expec}, putting
$$
\wh{\C{R}}(\C{T}) = \oint \C{G} \, \ud \B{G}_{\C{T}}(\C{G}).
$$
Let us remark that, according to \vref{lm:Average}, $\wh{\C{R}}(\C{T})$ can
be computed from repeated simulations from the distribution $\B{G}_{\C{T}}$.

\section{Comparison with other models}

\subsection{Comparison with Context Free Grammars}

Given a toric grammar $\C{G} \in \beta^*(\F{T})$, we may consider the split and merge 
process $G_t$ with reference grammar $\C{G}$ starting at $G_1 = 
\C{G}$ (so here we start at time $1$ with an initial state that is 
a grammar, instead of starting at time $0$ with an initial state that 
is a text). Due to the weak reversibility of \vref{pp:Markov},
$G_2$ almost surely falls in the same recurrent communicating class 
of $t \mapsto G_{2t}$, 
and the unique invariant probability measure supported by this 
recurrent communicating class defines 
a probability measure $\ov{\B{T}}_{\C{G}}$ on texts, and therefore a stochastic language
model. This way of defining the language generated by the grammar $\C{G}$
can be compared to the usual definition of the language generated 
by a Context Free Grammar. Indeed, the support of $\C{G}$ is a
Context Free Grammar, so this is meaningful to consider the language 
generated by this grammar and to compare it with the support of our 
stochastic language model.  

None of these two sets of sentences is contained in the other one.
In our stochastic model, the number of times a rule can be used 
is bounded, so if the recursive use of some rules is possible, 
the deterministic language will in this sense be larger.
On the other hand, the stochastic model uses both production 
and parsing to build new sentences, whereas the deterministic 
model uses only production rules. In this respect, the stochastic
model may, at least in some cases, define a much broader language, 
as we will show on the following example.

Let us take as dictionary the set
\[
  D =\{+,=\}\cup\llbracket1,N\rrbracket,
\]
where $\llbracket 1, N \rrbracket = \bigl\{ i \in \B{N}, 1 \leq i \leq N \bigr\}$,
and consider the toric grammar
\[
\C{G} = N^2 \otimes \; [_0 \; ]_N = N \; \oplus \; 
\bigoplus_{i=1}^N N \otimes \; [_i \; i \; \oplus  
\bigoplus_{i=2}^N N(i-1) \otimes \; [_i ~ ]_{i-1} + 1,
\]
and the text
\[
\C{T} = N \otimes \bigoplus_{i=1}^N \; [_0 \, i \underbrace{+ 1 + \dots + 1}_{N-i 
\text{ times }} = N.
\]
It is easy to check that $\B{T}_{\C{G}}(\C{T}) > 0$, (so that 
$\C{G} \in \beta^*(\C{T})$,)  that indeed
the support of $\C{T}$ is the language generated by $\supp(\C{G})$, 
seen as a Context Free Grammar, and that the stochastic language 
$\ov{\B{T}}_{\C{G}}$ generated by $\C{G}$ is able to produce 
with positive probability a set of sentences  
\[
\supp^2 ( \ov{\B{T}}_{\C{G}} ) \; \overset{\text{def}}{=}
\bigcup_{\C{T} \in \supp ( \ov{\B{T}}_{\C{G}})} \supp ( \C{T} ),
\]
equal to 
\begin{multline*}
\supp^2 ( \ov{\B{T}}_{\C{G}} )  = \Bigl\{ [_0 \; x_1 + \dots + x_i = 
x_{i+1} + \dots + x_j, \\ 1 \leq i < j \leq 2N, \; x_k \in 
\llbracket 1, N \rrbracket,  \; 1 \leq k \leq j, 
\; \sum_{k=1}^i x_k = \sum_{k=i+1}^j x_k = N \Bigr\}. 
\end{multline*}
Here, the number of sentences produced by the underlying 
Context Free Grammar is $\lvert \supp ( \C{T} ) \rvert = N$, 
whereas the number of sentences produced by our stochastic language 
model is $\lvert \supp^2 ( \ov{\B{T}}_{\C{G}} ) \rvert = 2^{2(N-1)}$.
Thus, in this small example based on arithmetic expressions (admittedly 
closer to a computing language than it is to a natural language), 
our new definition of the generated language induces a huge increase 
in the number of generated sentences. 

Note that with usual Context Free Grammar notations, $\supp(\C{G})$ would 
have been described as 
\begin{align*}
  \framebox{$0$}	&	\rightarrow \framebox{$N$} = N\\
  \framebox{$i$} 	&	\rightarrow i, & i = 1, \dots, N,\\
  \framebox{$i$}	&	\rightarrow \framebox{$i-1$} + 1, & i = 2, \dots, N,  
\end{align*}
where \framebox{$0$} is the start symbol and \framebox{$i$}\,, 
$i = 1, \dots, N$, are other non terminal symbols.

To count the number of elements in $\supp^2 ( \ov{\B{T}}_{\C{G}})$, 
one can remark that the number of ways $N$ can be written as 
$\sum_{k=1}^i x_k$ with an arbitrary number of terms is also 
the number of increasing integer sequences $0 < s_1 < \cdots < s_{i-1} < N$
of arbitrary length, which is also the number of subsets $\{s_1, \dots, s_{i-1} \}$ 
of $\{1, \dots, N-1\}$, that is $2^{N-1}$. 

Intuitively speaking, 
the underlying Context Free Grammar $\supp(\C{G})$ is limited to producing 
a small 
set of global expressions of the form $i+1+\ldots+1=N$, whereas the 
stochastic language model incorporates some crude logical reasoning  
that is capable of deducing from them a large set of new global expressions.

Let us remark also that, when we start as here from a text made of 
true arithmetic statements, the language generated by our language 
model is also made of true arithmetic statements. This shows that 
our approach to language modeling is capable of some sort of logical 
reasoning.

\subsection{Comparison with Markov models}

The kind of reasoning illustrated in the previous section is
related to the fact that we analyse global syntactic structures
represented by the global expressions of our toric grammars.

In order to give another point of comparison, we would 
like in this section to make a qualitative comparison 
with Markov models, that do not share this feature. To make a parallel 
between toric grammars and Markov models, we are going 
to show how a Markov model could be described in 
terms of toric grammars and label identification 
rules.

To build a Markov model in our framework, we 
have to use a deterministic splitting (or parsing) rule.
This is because in a Markov model, conditional probabilities 
are specified from left to right in a rigid data independent 
way. Let us introduce the Markov splitting rule
\begin{multline*}
\beta_m(\C{G}) = \bigl\{ \C{G}' \in \F{G}, 
\; \C{G}' = \C{G} \ominus 
[_0 \; a w \; ]_i \oplus [_0 \; a \; ]_j \oplus [_j \;
w \; ]_i, \\ i,j \in \B{N} \setminus \{ 0 \}, a \in D^+, w \in D, 
\C{G} \bigl( [_j \; S^* \bigr) = 0 \bigr\}. 
\end{multline*}
We will describe now label identification rules using 
concepts introduced in \vref{App:Average}.
Let us say that the pair of labels $p \in \bigl( \B{N} \setminus 
\{ 0 \} \bigr)^2$ is $\C{G}$-Markov if there is $w \in D$
such that $\C{G} \bigl( w ]_{p_1} S^* \bigr) \C{G} \bigl( 
w ]_{p_2} S^* \bigr) > 0$. Let us say that the sequence of pairs of 
labels $p_1, \dots, p_k$ is $\C{G}$-Markov if $p_j$ is 
$\xi_{p_1, \dots, p_{j-1}} \bigl( \C{G} \bigr)$-Markov.
It can be proved as in the case of congruent sequences that if $\sigma$ 
is a permutation and $p$ is $\C{G}$-Markov, then $p \circ \sigma$ 
is also $\C{G}$-Markov. It can also be proved that if $p$ and $q$ 
are maximal $\C{G}$-Markov sequences, then $\xi_{p} \equiv \xi_{q}$, 
and therefore $\xi_p (\C{G}) \equiv \xi_q(\C{G})$. We will call 
$\xi_p(\C{G}) \in \quotient{\F{G}}{\equiv}$ the Markov closure of $\C{G}$ and
use the notation $\xi_p(\C{G}) \overset{\text{def}}{\equiv}
\mu(\C{G})$,where $\mu(\C{G})$ is the Markov pendent of $\chi(\C{G})$ in the
construction of toric grammars. 

Let $S_t$, $0 \leq t \leq \tau$ be a splitting process based on 
the restricted splitting rule 
\[
\beta_r(\C{G}) = \bigl\{ \mu(\C{G'}), \; \C{G}' \in \beta_m(\C{G}) \bigr\}. 
\]   
It is not very difficult to check that the support of $S_{\tau}$ 
is contained in a single isomorphic class of grammars, so that, up 
to label remapping the result of this splitting process is 
deterministic. More specifically, starting from a text 
\[
\C{T} = \bigoplus_{j=1}^n ~ [_0 \, w_1^j \dots w_{\ell(j)}^j,
\]
where $w_i^j \in D \setminus \{ . \}$, $1 \leq i < \ell(j)$, $1 \leq j \leq n$, and 
$w_{\ell(j)}^j = .~$, $1 \leq j \leq n$ so that all sentences end 
with a period, we obtain a grammar isomorphic to 
\[
\C{G} = \bigoplus_{j=1}^n  \biggl( [_0 \; w_1^j \; ]_{w_1^j} ~ \bigoplus_{i=2}^{\ell(j)-1} ~~ 
[_{w_{i-1}^j} w_i^j ~ ]_{w_i^j} \oplus [_{w_{\ell(j)-1}^j} w_{\ell(j)}^j 
\biggr), 
\]
where we have used words as labels instead of integers, since in this 
model, due to the label identification rule, labels are functions of 
words (namely $]_{w}$ is the non terminal symbol following the word $w \in 
D$).

We can now define a Markov production mechanism, to replace the production 
process. It is described as a Markov chain $X_i$, $i \in \B{N}$, where 
$X_i \in D \cup \{ \Delta \}$, where $\Delta \notin D$ is a padding 
symbol used to embed finite sentences into infinite sequences of symbols, 
all equal to $\Delta$ for indices larger than the sentence length. 
The distribution of the Markov chain $X_i$ is as follows. Its 
initial distribution is 
\[
\B{P}(X_0 = w) = \frac{\C{G} \bigl( [_0 w ]_w \bigr) }{\C{G} \bigl( [_0 S^* 
\bigr)}, 
\]
and its transition probabilities are 
\begin{align*}
& \B{P} \bigl( X_i = \Delta \, | \, X_{i-1} = . \bigr) = 1,\\
& \B{P} \bigl( X_i = . \, | \, X_{i-1} = w \bigr) = \frac{ \C{G} \bigl( 
[_w \, . \bigr)}{\C{G} \bigl( [_w S^* \bigr)}, & w \in D \setminus \{ . \}  \\ 
& \B{P} \bigl( X_i = w' \, | \, X_{i-1} = w \bigr) = 
\frac{\C{G} \bigl( [_w w' ]_{w'} \bigr)}{ \C{G} \bigl( [_w S^* \bigr)}, 
& w, w' \in D \setminus \{ . \}.
\end{align*}
Roughly speaking, the difference with the production process $P_t$ defined 
previously is 
that in the production process the production rules are drawn at random 
without replacement whereas here, the production rules are drawn with 
replacement.

It is easy to see that the initial distribution and transition probabilities 
of the Markov chain $X_i$ are the empirical initial distribution and 
empirical transition probabilities of the training text $\C{T}$. 

In conclusion, to build a Markov model using the same framework as for 
toric grammars, we had to modify two steps in a dramatic way:
\begin{itemize}
\item we had to change the splitting process, and replace the random splitting 
process of toric grammars with a non random splitting process which chains 
forward transitions in a linear way;
\item we had to change in a dramatic way the label identification rule 
to replace the {\em forward and backward global condition} of 
toric grammars with a {\em backward only local condition}.
\end{itemize}
(The modification of the production process is less crucial and boils 
down to drawing production rules with or without replacement.)

We hope that this discussion of Markov models will help the reader 
realize that our model proposal is indeed really different from 
the Markov model at sentence level. We could have extended easily the discussion 
to Markov models of higher order, or to more general context tree 
models. We let the reader figure out the details. All these more 
sophisticated models show the same differences from toric grammars:
a more rigid splitting process and local backward label identification 
rules. 

\section{A small experiment}

Let us end this study with a small example. Here we use a small text that is 
meant to mimic what could be found in a tutorial to learn English as a foreign
language. We have added a more elaborate sentence at the end of the text to 
show its impact. More systematic experiments are yet to be carried out,
although the conception of this model was guided by experimental trial and 
errors with models starting with variable length Markov chains, before 
we tried global rules leading to grammars.

This is the training text $\C{T}$ (each line shows an expression, 
starting with its weight) :
{\small
\begin{verbatim}
1	[0 He is a clever guy . 
1	[0 He is doing some shopping . 
1	[0 He is laughing . 
1	[0 He is not interested in sports . 
1	[0 He is walking . 
1	[0 He likes to walk in the streets . 
1	[0 I am driving a car . 
1	[0 I am riding a horse too . 
1	[0 I am running . 
1	[0 Paul is crossing the street . 
1	[0 Paul is driving a car . 
1	[0 Paul is riding a horse . 
1	[0 Paul is walking . 
1	[0 Peter is walking . 
1	[0 While I was walking , I saw Paul crossing the street . 
\end{verbatim}
}
And now, the new sentences produced by the model 
( that is by $\wh{Q}_{\wh{\C{R}},\C{T}}$, approximated on $50$ 
iterations of the communication chain with kernel $q_{\wh{\C{R}}}$ ). 
{\small
\begin{verbatim}
1	[0 Paul is driving a car too . 
1	[0 Paul is doing some shopping . 
1	[0 Paul is laughing . 
1	[0 Paul is riding a horse too . 
1	[0 Paul is running too . 
1	[0 Paul is running . 
1	[0 Paul is not interested in sports too . 
1	[0 Paul is not interested in sports . 
1	[0 Paul is a clever guy too . 
1	[0 Paul is a clever guy . 
1	[0 Paul is walking too . 
1	[0 Peter is driving a car too . 
1	[0 Peter is driving a car . 
1	[0 Peter is doing some shopping . 
1	[0 Peter is laughing . 
1	[0 Peter is riding a horse too . 
1	[0 Peter is riding a horse . 
1	[0 Peter is running too . 
1	[0 Peter is running . 
1	[0 Peter is not interested in sports . 
1	[0 Peter is a clever guy . 
1	[0 Peter is crossing the street . 
1	[0 He is driving a car too . 
1	[0 He is driving a car . 
1	[0 He is riding a horse too . 
1	[0 He is riding a horse . 
1	[0 He is running too . 
1	[0 He is running . 
1	[0 He is not interested in sports too . 
1	[0 He is crossing the street too . 
1	[0 He is crossing the street . 
1	[0 He is walking too . 
1	[0 I am driving a car too . 
1	[0 I am doing some shopping . 
1	[0 I am laughing too . 
1	[0 I am laughing . 
1	[0 I am riding a horse . 
1	[0 I am not interested in sports . 
1	[0 I am a clever guy . 
1	[0 I am crossing the street too . 
1	[0 I am crossing the street . 
1	[0 I am walking too . 
1	[0 I am walking . 
1	[0 While I was driving a car , I saw Paul doing some shopping too . 
1	[0 While I was driving a car , I saw Paul doing some shopping . 
1	[0 While I was driving a car , I saw Paul riding a horse . 
1	[0 While I was driving a car , I saw Paul crossing the street . 
1	[0 While I was driving a car , I saw Paul walking . 
1	[0 While I was driving a car , I saw Peter riding a horse . 
1	[0 While I was doing some shopping , I saw Paul riding a horse . 
1	[0 While I was doing some shopping , I saw Paul walking . 
1	[0 While I was laughing too , I saw Peter crossing the street . 
1	[0 While I was laughing , I saw Peter riding a horse . 
1	[0 While I was riding a horse , I saw Paul driving a car too . 
1	[0 While I was riding a horse , I saw Paul driving a car . 
1	[0 While I was riding a horse , I saw Paul laughing . 
1	[0 While I was riding a horse , I saw Paul running . 
1	[0 While I was riding a horse , I saw Paul walking . 
1	[0 While I was riding a horse , I saw Peter not interested in sports . 
1	[0 While I was running , I saw Paul laughing . 
1	[0 While I was running , I saw Paul not interested in sports . 
1	[0 While I was running , I saw Paul a clever guy . 
1	[0 While I was running , I saw Paul walking . 
1	[0 While I was not interested in sports , I saw Paul driving a car . 
1	[0 While I was not interested in sports , I saw Paul riding a horse . 
1	[0 While I was a clever guy , I saw Paul running . 
1	[0 While I was a clever guy , I saw Paul crossing the street . 
1	[0 While I was a clever guy , I saw Paul walking . 
1	[0 While I was crossing the street , I saw Paul riding a horse . 
1	[0 While I was crossing the street , I saw Paul running . 
1	[0 While I was crossing the street , I saw Paul crossing the street . 
1	[0 While I was crossing the street , I saw Paul walking . 
1	[0 While I was crossing the street , I saw Peter walking . 
1	[0 While I was walking , I saw Paul driving a car . 
1	[0 While I was walking , I saw Paul laughing . 
1	[0 While I was walking , I saw Paul riding a horse . 
1	[0 While I was walking , I saw Paul running . 
1	[0 While I was walking , I saw Paul not interested in sports . 
1	[0 While I was walking , I saw Paul crossing the street too . 
1	[0 While I was walking , I saw Paul walking . 
1	[0 While I was walking , I saw Peter not interested in sports . 
1	[0 While I was walking , I saw Peter walking . 
\end{verbatim}
}
The reference grammar was learnt first, and was 
computed from $10$ samples of $\B{G}_{\C{T}}$. (We did not normalize the 
weights, since we were interested in the support of the local 
expressions only.)
{\small
\begin{verbatim}
10	[0 He likes to walk ]6 ]3 streets . 
2	[0 ]1 ]8 clever guy . 
2	[0 ]1 doing some shopping . 
2	[0 ]1 laughing . 
2	[0 ]1 not interested ]6 sports . 
2	[0 ]1 riding ]8 horse . 
2	[0 ]1 riding ]8 horse ]2 . 
2	[0 ]1 running . 
24	[0 ]7 am ]5 . 
28	[0 Paul is ]5 . 
40	[0 He is ]5 . 
4	[0 ]1 crossing ]3 street . 
4	[0 ]1 driving ]8 car . 
5	[0 ]4 is ]5 . 
6	[0 ]1 walking . 
7	[0 Peter is ]5 . 
8	[0 While ]7 was ]5 , ]7 saw ]4 ]5 . 
10	[1 He is 
2	[1 Peter is 
2	[1 While ]7 was ]5 , ]7 saw ]4 
6	[1 ]7 am 
8	[1 Paul is 
2	[2 too 
30	[3 the 
14	[4 Paul 
1	[4 Peter 
16	[5 crossing ]3 street 
16	[5 driving ]8 car 
16	[5 riding ]8 horse 
34	[5 walking 
8	[5 ]5 too 
8	[5 ]8 clever guy 
8	[5 doing some shopping 
8	[5 laughing 
8	[5 not interested ]6 sports 
8	[5 running 
20	[6 in 
50	[7 I 
50	[8 a 
\end{verbatim}
}

Although we did not yet make the software development effort required 
to test large text copora, we learnt a few interesting things from what 
we already tried:
\begin{itemize}
\item As it is, the model requires the inclusion of a sufficient number of simple and redundant sentences to start generalizing. At this stage, we do not 
know whether this could be avoided by changing the learning rules. We made 
quite a few attempts in this direction. All of them resulted in the 
production of grammatical nonsense. Breaking the global constraints
that are enforced by the model seems to have a dramatic
effect on grammatical coherence. This could be a clue that these global 
conservation rules reflect some fundamental feature of the syntactic 
structure of natural languages. Including a bunch of ``simple'' sentences 
made of frequent words may be seen as introducing a pinch of supervision 
in the learning process.  
\item The constraints on subexpressions frequencies in the learning rule 
\thmref{dfn:NarrowParse} and \ref{dfn:Innovation} were added to avoid some
unwanted generalizations. For instance here we took $\mu_1 \C{R} \bigl( 
[_0 \, S^* \bigr) = \mu_2 \C{R} \bigl( [_0 \, S^* \bigr) = 5$. If we had 
chosen $10$ instead of $5$, sentences of the kind
{\small
\begin{verbatim}
[0 While I was walking , I saw He crossing the street .
\end{verbatim}
}
would have emerged, where the pronoun ``He'' is substituted 
to a noun in the wrong place. We deliberately wrote the training 
text in such a way that ``He'' is more frequent than any noun, since we 
expect that to be true for any reasonable large corpus. Doing so, 
we were able to rule out the wrong construction by lowering the 
frequency constraint to avoid the unwanted substitution.
\item Despite all the limitations of this small example, 
it shows that the model is able to find out non trivial new constructs, like 
{\small
\begin{verbatim}
[0 While I was laughing too, I saw Peter crossing the street.
\end{verbatim}
}
where it has discovered that ``too'' could be added to the 
subordinate clause opening the sentence. We are quite pleased 
to see that such things could be learnt along very general 
label identification rules, while all the generalized 
sentences remain, if not all grammatically correct, at least
all grammatically plausible. Of course this judgement is 
purely subjective. But since we have no mathematical 
or otherwise quantitative definition of what natural languages
are, we have to be content with a subjective evaluation 
of models.
\end{itemize}
Studying how this learning model scales with large corpora is still a work
to be done (it will require from us that we optimize our code so that it 
can run efficiently on large data sets).

\section{Conclusion}

We have built in this paper a new statistical framework for the syntactic 
analysis of natural languages. 

The main idea pervading our approach is that trying to estimate the 
distribution of an isolated random sentence is hopeless. Instead we 
propose to build a Markov chain on sets of sentences (called texts 
in this paper), with non trivial recurrent communicating classes and to define 
our language model as the invariant measures of this Markov chain 
on each of these recurrent communicating classes. At each step, 
the Markov chain 
recombines the set of sentences constituting its current state, 
using cut and paste operations described by grammar rules.
In this way we define the probability 
distribution of an isolated random sentence only in an indirect way.
We replace the hard question of generating a random sentence by 
the hopefully simpler one of recombining a set of sentences in a
way that keep the desired distribution invariant.

\medskip

The strong points of our approach are
\begin{itemize}
\item a decisive departure from Markov models that are known to fail 
to catch the recursive structure of natural languages;
\item a new ``communication model'' concept that defines a Markov 
chain on texts and in parallel on toric grammars. This results in 
a new definition of the language generated by what appears as a 
weighted Context Free Grammar (called a toric grammar in the paper).
This new perspective on language production may help to overcome 
the challenge of weak stimulus learning;
\item in this respect, the split and merge process with reference grammar $\C{R}$ is the major mathematical achievement of the paper. It has non trivial 
mathematical properties proving that it can be simulated using a 
bounded number of operations at each step, and that the state space 
is divided into recurrent communicating classes each including a finite number 
of states;
\item preliminary experiments on small corpora are encouraging. They give 
the (acknowledgedly subjective) feeling that the model catches the structure 
of the natural languages we tried (French and English). Some inflection rules 
and other grammatical subtleties may be missed, but experimental outputs 
nevertheless give us the impression that we are heading in the right direction.
\end{itemize}

On the other hand, the model needs some refinements. In particular, our proposal to build a reference grammar from a text $\C{T} \in \F{T}$ through the 
grammar expectation
$$
\wh{\C{R}}(\C{T}) = \oint \C{G} \, \ud \B{G}_{\C{T}} \bigl( \C{G} \bigr)
$$
is clearly only a first foray into unknown territory. We hope to be 
able to elaborate more on this part of our research program in the 
future. 

\appendix

\section{Proofs}

\subsection{Bound on the length of splitting and production processes}\label{App:BoundSMProc}

\begin{proof}[Proof of \vref{pp:BoundSMProc}] 
  Let us define the length of an expression $e \in S^k \cap \C{E}$ as~$\ell(e) = k$. Let us introduce some remarkable weights associated with a grammar~$\C{G} \in \beta^*(\F{T})$. 
  \begin{align*}
    W_s(\C{G}) & = \sum_{e \in \C{E}} \C{G}(e),\\
    W_e(\C{G}) & = \sum_{e \in \C{E}} \C{G}(e) \ell(e)^{-1},\\
    W_l(\C{G}) & = \sum_{i = 1}^{+ \infty} \C{G}([_i \, S^*),\\
    W_w(\C{G}) & = \sum_{w \in D} \C{G}(wS^*).
  \end{align*}
  Let us define the set of canonical expressions as
  \[
    \C{E}_c = \C{E} \cap  \Biggl( \bigcup_{i \in \B{N}} \; [_i S^* \Biggr).
  \]
  Using previously introduced notations, we can write the grammar as 
  \[
    \C{G} = \sum_{e \in \C{E}_c} \C{G}(e) \otimes e.
  \]
  We will call this the canonical decomposition of~$\C{G}$. The two weights~$W_s(\C{G})$ and $W_e(\C{G})$ are better understood in terms of this canonical decomposition. They can be expressed as
  \begin{align*}
    W_s(\C{G}) & = \sum_{e \in \C{E}_c} \C{G}(e) \ell(e),\\
    W_e(\C{G}) & = \sum_{e \in \C{E}_c} \C{G}(e).
  \end{align*} 
  This shows that $W_s(\C{G})$ counts the ``number of symbols'' in the canonical decomposition of~$\C{G}$, whereas $W_e(\C{G})$ counts the number of expressions (that is $\C{G}(\C{E}_c)$, the weight put by the grammar on canonical expressions). We can also see from the definitions that $W_l(\C{G})$ counts the number of canonical expressions starting with a positive (that is non terminal) label,
that we will call for short the number of labels, and that $W_w(\C{G})$ counts the number of words. 

  Since a split increases the number of canonical expressions by one, the number of symbols in canonical expressions by two, the number of labels by one, and keeps the number of words constant, whereas a merge decreases these quantities in the same proportions, the following quantities are invariant in all the toric grammars involved: for any $\C{G} \in \F{G}$ such that $\sum_{t \in \B{N}} 
\B{P} \bigl( G_t = \C{G} \bigr) > 0$,
  \begin{align*}
    W_s(\C{G}) - 2 W_e(\C{G}) & = W_s(\C{T}) - 2 W_e(\C{T}),\\
    W_e(\C{G}) - W_l(\C{G}) & = W_e(\C{T}) - W_{l}(\C{T}) = W_e(\C{T}),\\
    W_w(\C{G}) & = W_w(\C{T}). 
  \end{align*}

  Moreover, for the same reasons, for any
$\C{T}' \in \F{T}$ and $\C{G} \in \F{G}$ such that \linebreak
$\sum_{t \in \B{N}} \B{P} \bigl( G_{2t} = \C{T}' \bigr) > 0$
and $\sum_{t \in \B{N}} \B{P} \bigl( G_{2t+1} = \C{G} \bigr) > 0$,
  \begin{gather*}
    \B{P} \Bigl( \tau = W_l(S_{\tau})  \, \bigl| \, S_0 = \C{T}' \Bigr) =  1, \\
    \B{P} \Bigl( \sigma = W_l(\C{G}) \, \bigl| \, P_0 = \C{G}, P_{\sigma} \in \F{T} \Bigr) = 1. 
  \end{gather*}
  Thus, we will prove the lemma if we can bound $W_l(\C{G})$ (or equivalently $W_l(S_{\tau})$ when $S_0 = \C{T}'$, since $S_{\tau}$ 
almost surely satisfies the conditions imposed on $\C{G}$). 
We can then remark that 
  \begin{align*}
    \sum_{e \in \C{E}_c} \C{G}(e) \B{1} \bigl[ \ell(e) \geq 3 \bigr] 
    & \leq \sum_{e \in \C{E}_c} \C{G}(e) \bigl[ \ell(e) - 2 \bigr] = W_{s}(\C{G}) - 2 W_e(\C{G}), \\
    \sum_{e \in \C{E}_c} \C{G}(e) \B{1} \bigl[ \ell(e) = 2 \bigr] & = \sum_{e \in \C{E}} \C{G}(e) \B{1} \bigl[ \ell(e) = 2 \bigr] \sum_{w \in D} \B{1} \bigl( e \in w S^* \bigr) \\ 
    & \leq \sum_{e \in \C{E}} \C{G}(e) \sum_{w \in D} \B{1} \bigl( e \in wS^* \bigr) = W_w(\C{G}),
  \end{align*}
  because any canonical expression of length~$2$ is of the form~$e = [_i w$, with~$i \in \B{N}$ and~$w \in D$, so that for any~$e \in \C{E}_c$ of length $2$, 
  \[
    \sum_{e' \in \F{S}(e)} \sum_{w \in D} \B{1} \bigl( e' \in wS^* \bigr) = 1.
  \]
  Thus 
  \[
    W_e(\C{G}) \leq W_w(\C{G}) + W_s(\C{G}) - 2 W_e(\C{G}),
  \]
  and consequently we can bound~$W_l(\C{G})$ by the split and merge invariant bound
  \[
    W_l(\C{G}) \leq W_l(\C{G}) - W_e(\C{G}) + W_w(\C{G}) + W_s(\C{G}) - 2 W_e(\C{G}).
  \]
  This, added to the fact that $W_l(\C{T}) = 0$ and $W_s(\C{T}) 
= W_w(\C{T}) + W_e(\C{T})$, proves that 
  \[
    W_l(\C{G}) \leq 2 \bigl[ W_w(\C{T}) - W_e(\C{T}) \bigr]. 
  \]
  This ends the proof, since $W_w(\C{T}) = \C{T}(DS^*)$ and $W_e(\C{T}) = \C{T}([_0 S^*)$.
\end{proof}

\subsection{Parsing Relations}\label{App:ParseRelations}

\begin{proof}[Proof of \vref{lm:ParseRelations}]
  The implication 
$$
\B{T}_{\C{G}}(\C{T}) > 0 \implies \B{G}_{\C{T}, \C{G}} (\C{G}) > 0
$$ 
is less trivial than it may seem. Indeed we can reverse the path of the splitting process~$S_t$, be it a parsing or a learning process, to obtain a path followed with positive probability by the production process, but reversing the production process does not give  a parsing process. Let us illustrate this difficulty on a simple example. Consider
  \[
    \C{T} = 1 \otimes [_0 a b c d \quad \text{ and } \quad \C{G} = [_0 a ]_1 \oplus [_1 b ]_2 \oplus [_2 c ]_3 \oplus [_3 d .
  \]  
  The production path 
  \[
    \C{G}, \quad [_0 \, a b \, ]_2 \oplus [_2 \, c \, ]_3 \oplus [_3 \, d, \quad [_0 \, a b \, ]_2 \oplus [_2 \, cd, \quad \C{T}
  \]
  has positive probability. The reverse path may have a positive probability for the learning process but not for the parsing process with reference~$\C{G}$, since none of the expressions~$[_0 \, ab \, ]_2$ or~$[_2 \, c d$ belongs to the support of~$\C{G}$. To parse~$\C{T}$ according to~$\C{G}$, one can instead follow with positive probability such a path as
  \[
    \C{T}, \quad [_0\, a b c \, ]_3 \oplus [_3 \, d, \quad [_0 \, a b \, ]_2 \oplus [_2 \, c \, ]_3 \oplus [_3 \, d, \quad \C{G}.
  \]
  To prove the lemma, we will have to show that it is always possible to find such an alternative parsing path. This property is fundamental to our approach, since it proves that the toric grammars we build can be used to parse the texts they can produce.

  Let us start with the easiest part of the proof. Assume that $\B{G}_{\C{T}, \C{R}}(\C{G}) > 0$. This means that there is a path~$\C{G}_0, \dots, \C{G}_k$ such that $\C{G}_0 = \C{T}$, $\C{G}_k = \C{G}$, and~$\C{G}_t \in \beta_{n}(\C{G}_{t-1}, \C{R})$. Anyhow it is easy to check that 
  \[
    \C{G}_t \in \beta_n(\C{G}_{t-1}, \C{R}) \implies \C{G}_{t-1} \in \alpha 
    (\C{G}_t),
  \]
  so that the reverse path is followed with a positive probability by the production process. This means that $\B{T}_{\C{G}}(\C{T}) > 0$. 

  In the case of the learning process, if $\B{G}_{\C{T}} \bigl( \C{G} \bigr) > 0$, there is a path~$\C{G}_t, 0 \leq t \leq k$, such that $\C{G}_t \in \beta_{\ell}(\C{G}_{t-1})$, $\C{G}_0 = \C{T}$ and $\C{G}_k = \C{G}$, consequently there is a label map~$f_t \in \F{F}$ such that $f_t(\C{G}_{t-1}) \in \alpha(\C{G}_t)$. We can then remark that 
  \[
    f_k \circ \cdots \circ f_t (\C{G}_{t-1}) \in \alpha \bigl( 
    f_k \circ \cdots \circ f_{t+1} ( \C{G}_t ) \bigr),
  \]
  because as already proved before in \vref{lm:MergeLabel}, $f \bigl( \alpha(\C{G}) \bigr) \subset \alpha \bigl( f( \C{G} ) \bigr)$. Let us consider the path~$\wt{\C{G}}_{t} = f_k \circ \cdots \circ f_{k-t+1} \bigl( \C{G}_{k-t} \bigr)$. It begins at~$\wt{\C{G}}_0 = \C{G}_k = \C{G}$ and ends at~$\wt{\C{G}}_k = f_k \circ \cdots \circ f_1 \bigl( \C{G}_0 \bigr) = \C{T}$. According to the previous remark, this path is followed by the production process with positive probability, proving that $\B{T}_{\C{G}}(\C{T}) > 0$. 

  Let us now come to the proof of the third implication of the lemma. For this let us assume now that $\B{T}_{\C{G}}(\C{T}) > 0$. Consider a path~$\C{G}_0, \dots, \C{G}_k$ such that $\C{G}_0 = \C{G}, \dots, \C{G}_k = \C{T}$ and $\C{G}_t \in \alpha(\C{G}_{t-1})$. 
  We are going to define some {\em decorated} path~$\Ct{G}_0, \dots, \Ct{G}_k$ with some added parentheses. Introduce a new set of symbols~$B = \bigl\{ (_i \, , \, )_i, i \in \B{N} \setminus \{0\} \bigr\}$ and assume that it is disjoint from the other symbols used so far, so that $B \cap S = \varnothing$. Consider the set of toric grammars~$\wt{\F{S}}$ based on the enlarged dictionary~$D \cup B$, and the projection~$\pi: \wt{\F{G}} \rightarrow \F{G}$ defined with the help of the canonical decomposition of toric grammars as
  \[
    \pi \biggl( \sum_{e \in \wt{\C{E}}_c} \C{G}(e) \otimes e \biggr) = 
    \sum_{e \in \wt{\C{E}}_c} \C{G}(e) \otimes \pi(e),   
  \]
  where $\wt{\C{E}}_c$ is the set of canonical expressions based on the enlarged dictionary~$D \cup B$, and where $\pi(e)$ is obtained by removing from the sequence of symbols~$e$ the symbols belonging to the decoration set~$B$ (that is the parentheses). 
  
  Let us put $\Ct{G}_0 = \C{G}$ and define $\Ct{G}_t$ for $t = 1, \dots, k$ by induction. We will check on the go that $\pi \bigl( \Ct{G}_t \bigr) = \C{G}_t$. It is obviously true for~$\Ct{G}_0$, because $\Ct{G}_0 \in \F{S}$, so that $\pi \bigl( \Ct{G}_0 \bigr) = \Ct{G}_0 = \C{G}_0$. That said, let us describe the construction of~$\Ct{G}_t$, assuming that $\Ct{G}_{t-1}$ is already defined, and satisfies $\pi \bigl( \Ct{G}_{t-1} \bigr) = \C{G}_{t-1}$. Consider the sequence of symbols~$a$ and~$b \in S^*$ and the index~$i \in \B{N} \setminus \{ 0 \}$ such that 
  \[
    \C{G}_t = \C{G}_{t-1} \oplus ab \ominus a \, ]_i \ominus [_i \, b.
  \]
  Since $\pi \bigl( \Ct{G}_{t-1} \bigr) = \C{G}_{t-1}$, and since $a \, ]_i \oplus [_i \, b \leq \C{G}_{t-1}$, there are $\wt{a} \in \wt{S}^*$ and $\wt{b} \in \wt{S}^*$ such that $\pi(\wt{a}) = a$, $\pi(\wt{b}) = b$, and $\wt{a} \, ]_i \oplus [_i \, \wt{b} \leq \Ct{G}_{t-1}$. (The choice of~$\wt{a}$ and~$\wt{b}$ may not be unique, in which case we can make any arbitrary choice). Let us define 
  \[
    \Ct{G}_t = \Ct{G}_{t-1} \oplus \wt{a} (_i \, \wt{b} \, )_i \ominus \wt{a} \, ]_i \ominus [_i \, \wt{b}.
  \]
  Since $\pi \bigl( \wt{a} (_i \, \wt{b} \, )_i \bigr) = \pi \bigl( \wt{a} \wt{b} \bigr) = a b$, 
  \[
    \pi \bigl( \Ct{G}_t \bigr) = \pi \bigl( \wt{G}_{t-1} \bigr) \oplus \pi \bigl( \wt{a} (_i \, b \, )_i \bigr) \ominus \pi \bigl( \wt{a} ]_i \bigr) \ominus \pi \bigl( [_i \wt{b} \bigr) = \C{G}_{t-1} \oplus ab \ominus a \, ]_i \ominus [_i \, b = \C{G}_t,
  \] 
  where we have used the obvious fact that $\pi$ is linear.

  We are now going to define another mapping between grammars that allows to recover~$\C{G}$ from any~$\Ct{G}_t$ (obviously the decorations where added to keep track of~$\C{G}$). Let us define $\psi: \wt{\F{S}}' \rightarrow \F{S}$ on the set of decorated grammars~$\wt{\F{S}}'$ which are supported by expressions where the parentheses~$(_i \, )_i$ are matched (at the same level) by the formula
  \[
    \psi\biggl( \sum_{e \in \wt{\C{E}}_c } \Ct{G}(e) \otimes e \biggr) 
    = \sum_{e \in \wt{\C{E}}_c } \Ct{G}(e) \psi(e),
  \]
  where $\psi(e)$ is defined by the rules
  \[
    \psi(e) = \begin{cases} 
      \psi \bigl( [_i a \, ]_j c \bigr) + \psi \bigl( [_j \, b \bigr), & \text{ if } e = [_i \, a (_j \, b )_j c, \text{ with } a, b, c \in \wt{S}^* \\
      \psi (e) = 1 \otimes e, &  \text{ otherwise.}
    \end{cases}
  \]
  It is easy to check that this definition is not ambiguous and that
  \[
    \psi(e) = \psi'(e) \oplus \bigoplus_{(_i a )_i \in \supp(e)} [_i \psi'(a), 
  \]
  where $\psi'(e)$ is the expression obtained from $e$ by replacing all the sequences between outer parentheses pairs~$(_j a )_j$ by~$]_j$. This is may be easier to grasp on some example:
  \[
  \psi \bigl( 
    [_0 a (_1 b (_2 c )_2 d )_1 e (_3 f (_4 g )_4 )_3 h \bigr) = 
    [_0 a \, ]_1 e \, ]_3 h \oplus [_1 \, b \, ]_2 d \oplus [_2 \, c \oplus [_3 \, f \, ]_4 \oplus [_4 \, g.
  \]
  It is easy to check by induction that $\Ct{G}_t \in \wt{\F{S}}'$. Let us check moreover that $\psi \bigl( \Ct{G}_t \bigr) = \C{G}$. Indeed $\psi \bigl( \Ct{G}_0 \bigr) = \psi \bigl( \C{G} \bigr) = \C{G}$ and  
  \[
    \psi \bigl( \wt{G}_t \bigr) = \psi \bigl( \wt{G}_{t-1} \bigr) \oplus \psi \bigl( \wt{a} (_i \wt{b} )_i \bigr) \ominus \psi \bigl( \wt{a} \, ]_i \bigr) \ominus \psi \bigl( [_i \, \wt{b} \bigr) 
    = \psi \bigl( \wt{G}_{t-1} \bigr),  
  \]
  since $\psi$ is linear and $\psi \bigl( \wt{a} (_i \, \wt{b} )_i \bigr) = \psi \bigl( \wt{a} \, ]_i \bigr) \oplus \psi \bigl( [_i \, b \bigr)$.
  
  We are now going to define a continuation for the path~$\bigl( \Ct{G}_t, 0 \leq t \leq k \bigr) $ that will bring us back to~$\C{G}$. 

  We will maintain during our inductive construction two properties:
  \begin{gather*}
    \psi \bigl(\Ct{G}_t \bigr)  = \C{G},\\
    \text{and } \quad \supp \bigl( \Ct{G}_t \bigr) \cap \wt{\C{E}}_{l} \subset \C{E}_l,
  \end{gather*}
  where $\wt{\C{E}}_l$ is the set of local decorated expressions, so that 
  \[
    \wt{\C{E}}_l = \bigl\{ e \in \wt{\C{E}}: [_0 \notin
    \supp(e) \bigr\}. 
  \]
  We already proved that the first property is satisfied by~$\Ct{G}_k$. As $\pi \bigl( \Ct{G}_k \bigr) = \C{G}_k = \C{T}$, $\supp(\Ct{G}_k) \cap \Ct{E}_l = \varnothing$, so that the second condition is also satisfied.
  Let us assume that, for some~$t > k$, $\Ct{G}_{t-1}$ has been defined and satisfies the two conditions above, and let us proceed to the construction of~$\Ct{G}_t$.
  
  As long as $\Ct{G}_{t-1} \not \in \F{S}$, (and this will be the case for~$t < 2k$), find some canonical expression~$e \in \wt{\C{E}}_c \setminus \C{E}_c$, such that $\Ct{G}_{t-1}( e ) \geq 1$. From our induction hypotheses, we see that necessarily $[_0 \in \supp(e)$. Our continuation will be such that each such expression has matching parentheses with matching labels, and we will check this on the go while building it by induction. Among those matching pairs of parentheses, there is necessarily at least one inner pair. We can for instance choose the one starting with the last opening parenthesis~$(_j$ of the sequence~$e$. This choice makes it obvious that the subsequence of~$e$ enclosed between~$(_j$ and~$)_j$ contains  no further parentheses. 

  Since $\psi$ is linear and preserves positive measures, $\C{G} \ominus \psi(e) = \psi\bigl( \Ct{G}_{t-1} \bigr) \ominus \psi(e) = \psi \bigl( \Ct{G}_{t-1} \ominus e \bigr) \geq 0$, On the other hand, $e$ has the form~$e = [_0 a (_j b )_j c$, where $\psi(b) = b$ (since $(_j )_j$ is an inner pair of parentheses in~$e$). As $\psi(e) = \psi \bigl( [_0 a \, ]_j c \bigr) + \psi \bigl( [_j b \bigr)$ and $\psi \bigl( [_j b \bigr) = [_j b$, this shows that $[_j \, b \leq \C{G}$, and therefore that $\C{G}( [_j b ) > 0$.
  Let us now define 
  \[
    \Ct{G}_t = \Ct{G}_{t-1} \ominus e \oplus [_j b \oplus [_0 \, a \, ]_j c.
  \]
  Applying $\psi$ to~$\Ct{G}_t$, we see as previously that $\psi \bigl( \Ct{G}_t \bigr) = \psi \bigl( \Ct{G}_{t-1} \bigr) = \C{G}$. As $\Ct{G}_k$ contains $k$~pairs of parentheses, and we consume one pair at each step~$t > k$, we see that $\Ct{G}_{2k}$ contains no more parentheses, so that $\Ct{G}_{2k} \in \F{G}$ and $\Ct{G}_{2k} = \psi \bigl( \Ct{G}_{2k} \bigr) = \C{G}$. Let us put now $\C{G}_t = \pi \bigl( \Ct{G}_t \bigr)$, for $t = k+1, \dots, 2k$. 
  We see that 
  \[
    \C{G}_t = \C{G}_{t-1} \ominus [_0 \, a b c \oplus [_0 \, a \, ]_j c \oplus [_j \, b,
  \]
  where $[_0 \, a b c, [_0 \, a \, ]_j c \in \C{E}$ and 
$\C{G} \bigl( [_j \, b \bigr) > 0$, so that 
$\C{G}_t \in \beta_n(\C{G}_{t-1}, \C{G})$, therefore $\C{G}_k = \C{T}, \dots, \C{G}_{2k} = \C{G}$ is a path of positive probability under the parsing process with reference~$\C{G}$, leading from~$\C{T}$ to~$\C{G}$, in other words, $\B{G}_{\C{T}, \C{G}}(\C{G}) > 0$ as required.  
\end{proof}

\subsection{Convergence to the expectation of a random toric grammar, proof of \vref{lm:Average}}\label{App:Average}

The proof of this results is based on the fact that the 
operation 
$$
\bigl( \C{G}, \C{G}' \bigr) \mapsto \chi \bigl( \C{G} \boxplus \C{G}' \bigr)
$$ is associative. 

Let us begin the proof by several definitions and lemmas.

For any grammar $\C{G} \in \F{G}$ and any pair of indices $p = (p^1, p^2)  \in 
\bigl( \B{N} \setminus 
\{ 0 \} \bigr)^2$, we will say that $p$ is $\C{G}$-congruent 
when there is $a \in S^*$ such that $\C{G}(a ]_{p^1} )
\C{G}(a \, ]_{p^2}) > 0$ or $\C{G}([_{p^1} \, a ) \C{G}([_{p^2}\, a) 
> 0$.

Let us define the label map $\xi_{p}$ as 
\[
\xi_p(i) = \begin{cases}
i, & \text{ when } i \notin \{p^1, p^2\},\\
\min \{p^1, p^2 \}, & \text{ when } i \in \{p^1, p^2\}.
\end{cases}
\]
For any sequence $p_1, \dots, p_k \in (\B{N} \setminus \{0\})^{2k}$ of 
pairs of indices, let us define the label map $\xi_{p_1, \dots, p_k}$ 
as 
\[
\xi_{p_1, \dots, p_k} = \xi_{\xi_{p_1, \dots, p_{k-1}}(p_k)} \circ 
\xi_{p_1, \dots, p_{k-1}},
\]
where $f\bigl((i,j)\bigr) 
= (f(i), f(j))$, for any $(i, j) \in \bigl( \B{N} \setminus \{0\} \bigr)^2$.

Let us say that $(p_1, \dots, p_k) \in \bigl( \B{N} \setminus \{ 0 \} 
\bigr)^{2k}$ is $\C{G}$-congruent, if $\xi_{p_1, \dots, p_{j-1}} (p_j)$ 
is $\xi_{p_1, \dots, p_{j-1}}\bigl( \C{G} \bigr)$-congruent for any 
$j \leq k$, and that it is maximal $\C{G}$-congruent 
is it is $\C{G}$-congruent and any $\C{G}$-congruent sequence 
of the form  $(p_1, \dots, p_k, p_{k+1})$ is such that 
\[
\xi_{p_1, 
\dots, p_{k}}(p_{k+1}^{1}) = \xi_{p_1, \dots, p_k}(p_{k+1}^{2}),
\]
or equivalently such that $\xi_{p_1, \dots, p_{k+1}} = \xi_{p_1, \dots, p_k}$.

\begin{lemma}\label{lm:MultipleXi}
For any sequence $(p_1, \dots, p_\ell) \in \Bigl( \B{N} \setminus \{ 0 \} 
\Bigr)^{2\ell}$, for any $k < \ell$,  
\[
\xi_{p_1, \dots, p_{\ell}} = \xi_{\xi_{p_1, \dots, p_k}(p_{k+1}, \dots, 
p_{\ell})} \circ \xi_{p_1, \dots, p_k}.
\]
\end{lemma}

\begin{proof}
By induction on $\ell$ for $k$ fixed. 
This is true from the definition for $\ell = k+1$. 
Assuming we have established the lemma for $\ell -1$, we can write
\begin{multline*}
\xi_{p_1, \dots, p_{\ell}} = \xi_{\xi_{p_1, \dots, p_{\ell-1}}(p_{\ell})} 
\circ \xi_{p_1, \dots, p_{\ell-1}} \\ 
= \xi_{\xi_{\xi_{p_1, \dots, p_k}(p_{k+1}, \dots, p_{\ell-1})} 
\circ \xi_{p_1, \dots, p_k} (p_{\ell})} \circ 
\xi_{\xi_{p_1, \dots, p_k}(p_{k+1}, \dots, p_{\ell-1})} \circ \xi_{
p_1, \dots, p_k} \\ 
= \xi_{\xi_{p_1, \dots, p_k}(p_{k+1}, \dots, p_{\ell})} \circ \xi_{p_1, \dots, 
p_k},
\end{multline*}
\end{proof}

\begin{lemma}\label{lm:XiPerm}
For any permutation $\sigma$ of $\{1, \dots, k\}$, 
\[
\xi_{p_1, \dots, p_k} \equiv \xi_{p_{\sigma(1)}, \dots, p_{\sigma(k)}}
\]
\end{lemma}

\begin{proof}
Let us consider the smallest equivalence relation containing the set \linebreak
$\{p_1, \dots, p_k\}$. Let $\ov{\pi}_k: \B{N} \setminus \{0\} \rightarrow
\C{C}$ be the corresponding projection of each label to its component.
Let us befine the label map $\pi_k$ by $\pi_k(0) = 0$ and 
\[
\pi_k(i) = \min \ov{\pi}_k(i), \qquad i > 0.
\]
We are going to prove by induction on $k$ that $\xi_{p_1, \dots, p_k} 
\equiv \pi_k$. Since $\pi_k$ is invariant by permutation of the sequence
$(p_1, \dots, p_k)$, this will prove the lemma.

Let us remark now that $\xi_{p_1, \dots, p_k} \equiv \pi_k$ if and only 
if 
\[
\xi_{p_1, \dots, p_k}(i) = \xi_{p_1, \dots, p_k}(j) 
\iff \pi_k(i) = \pi_k(j), \qquad i, j > 0.
\]
So we are going to prove this equivalence. 
It is easy to see from the previous lemma 
that for any integer $m = 1, \dots, k$,  
\begin{equation}
\label{eq:equality}
\xi_{p_1, \dots, p_k} (p_{m}^{1}) = \xi_{p_1, \dots, p_{k}}(p_{m}^{2}). 
\end{equation}
Indeed, 
\[
\xi_{p_1, \dots, p_k} =  \xi_{\xi_{p_1, \dots, 
p_{m}}(p_{m+1}, \dots, p_k)} \circ \xi_{\xi_{p_1, \dots, p_{m-1}}(p_{m})} \circ \xi_{p_1, \dots, p_{m-1}},
\]
so that, changing $p_m$ for $\xi_{p_1, \dots, p_{m-1}}(p_m)$, we are 
back to proving the result when $m = k = 1$, where it is obvious 
from the definitions.

Now, \vref{eq:equality} and the minimality of $\pi_k$ implies 
that 
\[
\pi_k(i) = \pi_k(j) \implies \xi_{p_1, \dots, p_k} (i) = 
\xi_{p_1, \dots, p_k}(j), \qquad i, j > 0.
\]
Let us assume conversely that $\xi_{p_1, \dots, p_k}(i) = 
\xi_{p_1, \dots, p_k}(j)$ and let 
\[
m = \min \bigl\{ \ell: \xi_{p_1, \dots, p_{\ell}}(i)
= \xi_{p_1, \dots, p_{\ell}}(j) \bigr\}.
\]
Since $\xi_{p_1, \dots, p_m} = \xi_{\xi_{p_1, \dots, p_{m-1}}(p_m)} 
\circ \xi_{p_1, \dots, p_{m-1}}$, we see that necessarily 
\[
\xi_{p_1, \dots, p_{m-1}} \bigl( \{ i, j \} \bigr) 
= \xi_{p_1, \dots, p_{m-1}} \bigl( \{ p_{m}^{1}, p_{m}^{2} \} \bigr),
\]
and that this set contains two distinct elements.
Exchanging the role of $i$ and $j$ if necessary, we can assume 
without loss of generality that 
\[
\xi_{p_1, \dots, p_{m-1}} \bigl( (i, j) \bigr) 
= \xi_{p_1, \dots, p_{m-1}} \bigl( p_m \bigr). 
\]
From the induction hypothesis, this implies that 
$\pi_{m-1} \bigl( (i,j) \bigr) = \pi_{m-1}( p_m )$. 
Since the equivalence relation defined by $\pi_{m-1}$ is a subset 
of the equivalence relation defined by $\pi_k$, this implies 
that $\pi_k \bigl( (i, j) \bigr) = \pi_k( p_m )$. Since moreover 
$\pi_k(p_{m}^{1}) = \pi_k (p_{m}^{2})$, this implies that 
$\pi_k(i) = \pi_k(j)$. 
\end{proof}

\begin{lemma}\label{lm:XiRelabel}
For any $f \in \F{F}$, any sequence of pairs of positive labels 
$p_1, \dots, p_k$, there is a label map $g \in \F{F}$ such that 
\[
\xi_{f(p_1, \dots, p_k)} \circ f = g \circ \xi_{p_1, \dots, p_k}.
\]
\end{lemma}

\begin{proof}
We have to prove that 
\[
\xi_{p_1, \dots, p_k}(i) = \xi_{p_1, \dots, p_k}(j) 
\implies \xi_{f(p_1, \dots, p_k)} \circ f (i) = \xi_{f(p_1, 
\dots, p_k)} \circ f(j), \qquad i, j > 0.
\]
From the proof of the previous lemma, it is enough to 
check that the right-hand side holds when $(i,j) = p_m$, $m=1, 
\dots, k$, which is then obvious. 
\end{proof}

\begin{lemma}\label{lm:Congruence}
If $f \in \F{F}$ and $(p_1, \dots, p_k)$ is $\C{G}$-congruent, then $(f(p_1), \dots, f(p_k))$ is also $f(\C{G})$-congruent.
\end{lemma}

\begin{proof}
Assume that for some $a \in S^*$ 
\[
\xi_{p_1, \dots, p_{m-1}} \bigl( \C{G} \bigr) \bigl( a \, ]_{ \xi_{p_1, \dots, 
p_{m-1}}(p_{m}^{1})} \bigr) > 0.
\]
Then, $\xi_{f(p_1, \dots, p_{m-1})} \circ f = g \circ \xi_{p_1, \dots, p_{m-1}}$,  and 
\begin{multline*}
\xi_{f(p_1, \dots, p_{m-1})} \circ f \bigl( \C{G} \bigr) 
\bigl( g(a) \, ]_{\xi_{f(p_1, \dots, p_{m-1})} \circ f (p_{m}^{1})} \bigr) 
\\ = g \circ \xi_{p_1, \dots, p_{m-1}} \bigl( \C{G} \bigr)  
\bigl( g (a) \, ]_{g \circ \xi_{p_1, \dots, p_{m-1}}(p_{m}^{1})} \bigr)  
\\ = \xi_{p_1, \dots, p_m} \bigl( \C{G} \bigr) 
\Bigl( g^{-1} \circ g \bigl( a \, ]_{\xi_{p_1, \dots, p_{m-1}}(p_{m}^{1})} 
\bigr) \Bigr) \\ \geq \xi_{p_1, \dots, p_{m-1}} \bigl( \C{G} \bigr) 
\bigl( a \, ]_{\xi_{p_1, \dots, p_{m-1}}(p_{m}^{1})} \bigr)
> 0.
\end{multline*}
The same is true when $p_{m}^{1}$ is replaced with $p_{m}^{2}$ and when 
$a\,]_{\xi_{p_1, \dots, p_{m-1}(p_{m}^{1})}}$ is replaced with 
$a [_{\xi_{p_1, \dots, p_{m-1}(p_{m}^{2})}}$. 

The lemma is a straightforward consequence of these remarks and 
the definition of a congruent sequence. 
\end{proof}

\begin{lemma}\label{lm:AssocCongruence}
If $(p_1, \dots, p_k)$ and $(q_1, \dots, q_{\ell})$ are both 
$\C{G}$-congruent, then 
\[
(p_1, \dots, p_k, q_1, \dots, q_{\ell}) 
\]
is $\C{G}$-congruent.
\end{lemma}

\begin{proof}
According to the previous lemma, $\xi_{p_1, \dots, p_k}(q_1, \dots, q_{\ell})$
is $\xi_{p_1, \dots, p_k} \bigl( \C{G} \bigr)$-congruent. 
Coming back to the definition this proves that 
\[
\xi_{\xi_{p_1, \dots, p_k}(q_1, \dots, q_{\ell-1})} 
\circ \xi_{p_1, \dots, p_k}(q_\ell)
\]
is 
\[
\xi_{\xi_{p_1, \dots, 
p_k}(q_1, \dots, q_{\ell-1})} \circ \xi_{p_1, \dots, p_k} \bigl( \C{G} \bigr)\text{-congruent}. 
\]
In \vref{lm:MultipleXi} we have moreover proved that
\[
\xi_{p_1, \dots, p_k, q_1, \dots, q_{\ell-1}} = \xi_{\xi_{p_1, 
\dots, p_k}(q_1, \dots, q_{\ell-1})} \circ \xi_{p_1, \dots, 
p_k}.
\]
This identity applied to the above statement shows that 
$(p_1, \dots, p_k, q_1, \dots, q_\ell)$ satisfies the definition 
of a $\C{G}$-congruent sequence.
\end{proof}

\begin{prop}\label{pp:ChiXi}
If $(p_1, \dots, p_k)$ and $(q_1, \dots, q_{\ell})$ are both 
maximal $\C{G}$-congruent, then 
\[
\xi_{p_1, \dots, p_k} \bigl( \C{G} \bigr) \equiv \xi_{q_1, \dots, 
q_{\ell}} \bigl( \C{G} \bigr) \equiv \chi \bigl( \C{G} \bigr).
\]
\end{prop}

\begin{proof}
From the previous lemma, $(p_1, \dots, p_k, q_1, \dots, q_{\ell})$ 
is $\C{G}$-congruent. Since $p$ is maximal, $\xi_{p_1, \dots, p_k, 
q_1, \dots, q_\ell} = \xi_{p_1, \dots, p_k}$. 
In the same way $\xi_{q_1, \dots, q_{\ell}, p_1, \dots, p_k} = 
\xi_{q_1, \dots, q_\ell}$. 
We have seen moreover in a previous lemma that 
\[
\xi_{p_1, \dots, p_k, q_1, \dots, q_{\ell}} 
\equiv
\xi_{
q_1, \dots, q_{\ell}, p_1, \dots, p_k }. 
\]
This proves that $\xi_{p_1, \dots, p_k} \equiv \xi_{q_1, \dots, q_{\ell}}$. 

We see from the 
definition of $\chi$ (see \vref{dfn:LabelIdentification}) that there is some maximal $\C{G}$-congruent 
sequence $r_1, \dots, r_m$ such that $\chi \bigl( \C{G} \bigr) 
= \xi_{r_1, \dots, r_m} \bigl(\C{G} \bigr)$. Therefore 
\[
\chi \bigl( \C{G} \bigr)  
\equiv \xi_{p_1, \dots, p_k} \bigl( \C{G} \bigr) 
\equiv \xi_{q_1, \dots, q_{\ell}} \bigl( \C{G} \bigr). 
\]
\end{proof}

\begin{prop}\label{pp:Associativity}
For any $\C{G}, \C{G}' \in \F{G}$, 
\[
\chi \bigl( \chi ( \C{G} )  
\boxplus \C{G}' \bigr) =  \chi \bigl( 
\C{G} \boxplus \C{G}' \bigr).
\]
Consequently, for any $\C{G}, \C{G}', \C{G}'' \in \F{G}$, 
\[
\chi \Bigl( \chi \bigl( \C{G} \boxplus \C{G}' \bigr) \boxplus 
\C{G}'' \Bigr) = \chi \Bigl( \C{G} \boxplus \C{G}' \boxplus 
\C{G}'' \Bigr).
\]
\end{prop}

\begin{proof}
Let us assume that $\C{G}$, $\C{G}'$ and $\chi(\C{G})$ use disjoint label 
sets, so that 
\begin{align*}
\chi \bigl( \chi \bigl( \C{G} \bigr) \boxplus 
\C{G}' \bigr) & \equiv 
\chi \bigl( \chi (\C{G}) + 
\C{G}' \bigr), \\
\chi \bigl( \C{G} \boxplus \C{G}' \bigr) & \equiv 
\chi \bigl( \C{G} + \C{G}' \bigr).
\end{align*}

Let $p_1, \dots, p_k$ be some maximal $\C{G}$-congruent sequence. 
It is also $\bigl(\C{G} + \C{G}'\bigr)$-congruent, and since 
label sets are disjoint,
\[
\xi_{p_1, \dots, p_k} \bigl( \C{G} \bigr)  
+ \C{G}' =  \xi_{p_1, \dots, p_k} \bigl( 
\C{G} + \C{G}'  \bigr). 
\]

Let us continue the sequence $p_1, \dots, p_k$ to form 
a maximal $\bigl( \C{G} + \C{G}' \bigr)$-congruent 
sequence $p_1, \dots, p_{\ell}$. 
Let $(q_{k+1}, \dots, q_\ell)$ be defined as  
\[
q_{m} = \xi_{p_1, \dots, p_k}(p_{k+m}).
\]
We see from the definitions that $(q_{k+1}, \dots, q_{\ell})$ 
is a maximal $\xi_{p_1, \dots, p_k} \bigl( \C{G} 
+ \C{G}' \bigr)$-congruent sequence, and therefore 
a maximal $\Bigl(\xi_{p_1, \dots, p_k}(\C{G}) + \C{G}' \Bigr)$-congruent 
sequence. Consequently 
\begin{multline*}
\chi \Bigl( \chi \bigl( \C{G} \bigr) 
+ \C{G}' \Bigr) \equiv
\xi_{q_{k+1}, \dots, q_{\ell}} \Bigl( \xi_{p_1, 
\dots, p_k}  \bigl( \C{G} \bigr) 
+ \C{G}' \Bigr) \\ 
= \xi_{q_{k+1}, \dots, q_\ell} \circ \xi_{p_1, \dots, p_k} \bigl( 
\C{G} + \C{G}' \bigr) = \xi_{\xi_{p_1, \dots, p_k} (p_{k+1}, \dots, p_{\ell})} 
\circ \xi_{p_1, \dots, p_k} \bigl( \C{G} + \C{G}' \bigr) \\  
= \xi_{p_1, \dots, p_{\ell}} \bigl( 
\C{G} + \C{G}' \bigr) \equiv \chi \Bigl( 
\C{G} + \C{G}' \Bigr),
\end{multline*}
proving the proposition.
\end{proof}

\begin{proof}[Proof of \vref{lm:Average}]

Let $\pi$ be the projection of $\F{G}$ on $\quotient{\F{G}}{\equiv}$.

From the law of large numbers, we have that, for all $\C{G}\in\F{G}$,
\[
  \frac{1}{n}\sum_{i=1}^n\indi(G_i\equiv\C{G}) \underset{n\rightarrow\infty}{\longrightarrow}\B{G}(\pi(\C{G})).
\]
Let us now remark that
 $\ds \bigboxplus_{i=1}^n n^{-1}G_i	=\bigboxplus_{\ov{\C{G}}\in\quotient{\F{G}}{\equiv}}\bigboxplus_{\substack{i\\G_i\in\ov{\C{G}}}}n^{-1}G_i$.
Thus  
\begin{multline*}
\frac{1}{n} \chi\left(\bigboxplus_{i=1}^n G_i\right)	=\chi\left(\bigboxplus_{\ov{\C{G}}\in\quotient{\F{G}}{\equiv}}\chi\left(\bigboxplus_{\substack{i\\G_i\in\ov{\C{G}}}}n^{-1}G_i\right)\right) \\
  =\chi\left(\bigboxplus_{\ov{\C{G}}\in\quotient{\F{G}}{\equiv}}\left(\sum_{i=1}^nn^{-1}\indi(G_i\in\ov{\C{G}})\right)\chi(\ov{\C{G}})\right)	\\
  =\chi\left(\bigboxplus_{\ov{\C{G}}\in\quotient{\F{G}}{\equiv}}\left(\sum_{i=1}^nn^{-1}\indi(G_i\in\ov{\C{G}})\right)\ov{\C{G}}\right).
\end{multline*}
We used here \vref{pp:Associativity}  
and the fact that for any $a, b \in \B{R}_+$, 
$$
\chi \bigl[ (a \, \C{G})\boxplus (b \, \C{G}) 
\bigr]=(a + b) \,\chi(\C{G}),
$$
 which comes from the following reasoning:
Suppose that 
$$
\{1,\ldots,d\}=\{i\,;\ \C{G}([_iS^*)>0\},
$$
 and let~$p_i=(2i,2i-1)$. Since each~$p_i$ is $ (a \, \C{G}) 
\boxplus ( b \, \C{G}) $-congruent, $(p_1,\ldots,p_d)$ is also 
$(a \, \C{G}) \boxplus (b \, \C{G})$-congruent, from \vref{lm:AssocCongruence}. It is quite straightforward to see that
\[
  \xi_{p_1,\ldots,p_d} \bigl[ (a \, \C{G} ) \boxplus (b \, \C{G}) \bigr] \equiv
(a + b) \, \C{G}.
\]
This implies that
$$
  \chi \bigl[ ( a \, \C{G} ) \boxplus ( b \, \C{G}) \bigr]  =
\chi \circ \xi_{p_1,\ldots,p_d} \bigl[ ( a \, \C{G} ) \boxplus ( b \, \C{G}) \bigr] 	
    = \chi \bigl[ (a + b) \, \C{G} \bigr] 	= (a + b) \,\chi(\C{G}).	
$$

To take the limit inside $\chi$, we need to prove that $\chi$ is continuous
in a suitable sense. 
Actually, $\cal{G}\mapsto \chi(\cal{G})$ is continuous on sets of fixed 
support, and this is what is required to conclude.

Indeed, for any sequence $(\C{G}_i)$ with fixed support for $n$ large enough, there is a fixed label map $f$ (depending on the support) such that for 
$n$ large enough $\chi(\C{G}_i)=f(\C{G}_i)$, and the result follows from the fact that $\C{G}\mapsto f(\C{G})$ is continuous; since $f(\C{G})(A)=\C{G}(f^{-1}(A))$.

Consequently
\begin{multline*}
  \lim_{n \rightarrow \infty} \frac{1}{n} \; \chi \Biggl( \bigboxplus_{i=1}^n 
G_i	 \Biggr)  = 
\chi \Biggl( \bigboxplus_{\ov{\C{G}} \in \quotient{\F{G}}{\equiv}} 
\lim_{n \rightarrow \infty} \biggl( \frac{1}{n} \sum_{i=1}^n \B{1} \bigl( 
G_i \in \ov{\C{G}} \bigr) \biggr) \ov{\C{G}} \Biggr) \\ 
\begin{aligned} & = 
\chi\left(\bigboxplus_{\ov{\C{G}}\in\quotient{\F{G}}{\equiv}}\B{G}(\ov{\C{G}})\ov{\C{G}}\right) 
    =\chi\left(\bigboxplus_{\ov{\C{G}}\in\quotient{\F{G}}{\equiv}}\B{G}(\ov{\C{G}})\chi(\ov{\C{G}})\right)	\\
    & =\chi\left(\bigboxplus_{\ov{\C{G}}\in\quotient{\F{G}}{\equiv}}\bigboxplus_{\C{G}\in\ov{\C{G}}}\B{G}(\C{G})\chi(\C{G})\right)	
    =\chi\left(\bigboxplus_{\C{G}\in\F{G}}\B{G}(\C{G})\chi(\C{G})\right)
\end{aligned}	\\
    =\chi\left(\bigboxplus_{\C{G}\in\F{G}}\B{G}(\C{G})\C{G}\right)	
    =\oint \C{G} \, \ud \B{G}(\C{G}).
\end{multline*}
\end{proof}

\section{Language produced by a Toric Grammar}\label{App:DirectProduction}

In this appendix, we make a deterministic study of the language produced 
by a toric grammar $\C{G} \in \beta^*(\F{T})$. More precisely, we are 
interested in the support of the distribution $\B{T}_{\C{G}}$ 
of the final state of the production process. 

\begin{lemma}\label{lm:NbSplit}
  Let~$\C{T} \in \F{T}$ be some text and $\C{G} \in \beta^*(\C{T})$ be some grammar obtained by splitting this text a finite number of times. The number of splits performed can be read in~$\C{G}$ and is equal to
  \[
    n = \sum_{i=1}^{+ \infty} \C{G} \bigl(\, ]_i S^* \bigr).
  \]
  Let us put $\ov{\alpha}(\C{G}) = \alpha^n(\C{G})$. Then, $\C{T} \in \ov{\alpha}(\C{G})\subset \F{T}$, moreover $\ov{\alpha}(\C{G}) = 
\supp (\B{T}_{\C{G}})$. 
\end{lemma}

\begin{proof}
  The grammar $\C{G}$ is obtained by making a succession of splits. Each of those splits add one~$[_i$ and one~$]_i$ to the grammar, whereas in the original text there are no $[_i$ nor $]_i$, except for the $[_0$ at the beginning of each sentence. Since application of an element of~$\F{F}$ does not change the number of such symbols, they may be used to count the number of splits performed.

  Let us take then a sequence of toric grammars~$\C{T} = \C{G}_0 , \dots, \C{G}_n = \C{G}$, such that $\C{G}_k \in \beta(\C{G}_{k-1})$. From \vref{lm:Merge}, there is a sequence~$f_1, \dots, f_n \in \F{F}$ such that $f_k \bigl( \C{G}_{k-1} \bigr) \in \alpha \bigl( \C{G}_k \bigr)$. Let us prove by induction that for any~$k = 0, \dots, n$,
  \[
    f_k \circ \cdots \circ f_1 \bigl( \C{T} \bigr) \in  \alpha^{k} \bigl( \C{G}_{k} \bigr).
  \]
  Indeed, this is true for~$k = 0$, since $\C{G}_0 = \C{T}$. 
  Moreover, assuming that the assertion holds for $k-1$, we deduce that
\[
f_k \circ \dots \circ f_1 \bigl( \C{T} \bigr) \in  f_k \Bigl( 
\alpha^{k-1}\bigl( \C{G}_{k-1} \bigr) \Bigr)
\subset \alpha^{k-1} \Bigl( f_k \bigl( \C{G}_{k-1} \bigr) \Bigr) 
\subset \alpha^k \bigl( \C{G}_k \bigr).
\]
  showing that if the assertion holds for~$k$, it also holds for~$k+1$. For~$k = n$, we obtain that 
  \[
  f_n \circ \cdots \circ f_1 \bigl( \C{T} \bigr) \in \alpha^{n} \bigl( \C{G}_n \bigr).
  \]
  As $f_n \circ \cdots \circ f_1 \bigl( \C{T} \bigr) = \C{T}$, since $\C{T}$ is a text, and $\C{G}_n = \C{G}$, we get that $\C{T} \in \alpha^n \bigl( \C{G} \bigr)$. 

  Let us consider now $\C{G}'\in\ov{\alpha}(\C{G})$. Let $(\C{G}=\C{G}_0,\ldots,\C{G}_n=\C{G}')$ the chain of grammars leading to~$\C{G}'$. Then for any~$k = 0, \dots, n$,
  \[
    \sum_{i=1}^{+ \infty} \C{G}_k \bigl(\, ]_i S^* \bigr)=n-k,
  \]
  since $\C{G}_k\in\alpha(\C{G}_{k-1})$ and each merge takes away one~$[_i$ and one~$]_i$. This implies that $\sum_{i=1}^{+ \infty} \C{G}' \bigl(\, ]_i S^* \bigr)=0$, and thus $\C{G}'\in\F{T}$.
  
  Note that, as remarked  above, repeated merges may create elements of the type~$[_i \, a \, ]_i b$. However, this will not happen if $n$ successful merges can be performed. Indeed in the case when expressions of the form~$[_i \, a \, ]_i b$ remain unmatched during the merge process, we will get $\alpha \bigl( \C{G}_k \bigr) = \varnothing$ for some~$k < n$.
\end{proof}


\begin{thebibliography}{20}
\bibitem{Baker79}
  \textsc{Baker, J.K.} (1979).
  Trainable grammars for speech recognition.
  \textit{The Journal of the Acoustical Society of America},
  \textbf{65}(S1) S132--S132.
\bibitem{Chi98}
  \textsc{Chi, Z.} and \textsc{Geman, S.}  (1998).
  Estimation of probabilistic context-free grammars.
  \textit{Computational Linguistics},
  \textbf{24}(2) 299--305.
  MIT Press.
\bibitem{Chi99}
  \textsc{Chi, Z.} (1999).
  Statistical properties of probabilistic context-free grammars.
  \textit{Computational Linguistics},
  \textbf{25}(1) 131--160.
  MIT Press.
\bibitem{Chomsky56}
  \textsc{Chomsky, N.} (1956).
  Three Models for the Description of Language.
  \textit{IRE Transactions on Information Theory}
  \bibitem{Chomsky57}
  \textsc{Chomsky, N.} (1957).
  \textit{Syntactic Structures}.
  Mouton \& Co.
\bibitem{Chomsky65}
  \textsc{Chomsky, N.} (1965).
  \textit{Aspects of the Theory of Syntax.}
  MIT Press.
\bibitem{Chomsky95}
  \textsc{Chomsky, N.} (1995).
  \textit{The Minimalist Program.}
  MIT Press.
\bibitem{Cohen12}
  \textsc{Cohen, S. B.} and \textsc{Smith, N. A.} (2012).
  Empirical Risk Minimization for Probabilistic Grammars:
  Sample Complexity and Hardness of Learning.
  \textit{Computational Linguistics}.
  \textbf{38}(3) 479--526.
\bibitem{Pietra94}
  \textsc{Della Pietra, S.}, \textsc{Della Pietra, V.}, 
\textsc{Gillett, J.}, \textsc{Lafferty, J.}, \textsc{Printz, H.} and \textsc{Ure{\v{s}}, L.}.
  Inference and estimation of a long-range trigram model.
  \textit{Grammatical Inference and Applications},
  78--92.
  Springer.
\bibitem{Lari90}
  \textsc{Lari, K.} and \textsc{Young, S.} (1990).
  The estimation of stochastic context-free grammars using the inside-outside algorithm.
  \textit{Computer speech \& language},
  \textbf{4}(1) 35--56.
  Elsevier.
\bibitem{Norris}
  \textsc{Norris, J. R.}
  (1998)
  \textit{Markov Chains}.
  Cambridge University Press.
\bibitem{Roark01}
  \textsc{Roark, B.} (2001).
  Probabilistic Top-Down Parsing and Language Modeling.
  \textit{Computational Linguistics}.
  \textbf{27}(2) 249--276.
\bibitem{Sakakibara90}
  \textsc{Sakakibara, y.} (1990).
  Learning context-free grammars from structural data in polynomial time.
  \textit{Theoretical Computer Science},
  \textbf{76} (2) 223--242.
  Elsevier.
\bibitem{Stabler09}
  \textsc{Stabler, E.} (2009)
  Mathematics of language learning.
  \textit{Histoire, Épistémologie, Langage}
  \textbf{31}(1) 127--145.
\bibitem{Tan12}
  \textsc{Tan, M., Zhou, W., Zheng, L.} and \textsc{Wang, S.} (2012)
  A Scalable Distributed Syntactic, Semantic, and Lexical
  Language Model.
  \textit{Computational Linguistics}.
  \textbf{38}(3) 631--671.
\end{thebibliography}
\end{document}